\documentclass{article}

\usepackage[nonatbib, preprint]{neurips_2021}
\usepackage[numbers]{natbib}

\usepackage[utf8]{inputenc} %
\usepackage[T1]{fontenc}    %
\usepackage[colorlinks=true,linkcolor=blue,citecolor=blue]{hyperref}
\usepackage{url}            %
\usepackage{booktabs}       %
\usepackage{amsfonts}       %
\usepackage{nicefrac}       %
\usepackage{microtype}      %
\usepackage{xcolor}         %

\usepackage{amsmath, amssymb}
\usepackage{graphicx}

\renewcommand{\bar}{\overline}

\usepackage{amsmath,amsthm,amsfonts,bm}

\newtheorem{proposition}{Proposition}[section]

\newtheorem{principle}{Remark}[section]

\theoremstyle{definition}
\newtheorem*{proposition*}{Proposition}

\def\eqref#1{equation~\ref{#1}}

\def\1{\bm{1}}

\def\eps{{\epsilon}}

\DeclareMathAlphabet{\mathsfit}{\encodingdefault}{\sfdefault}{m}{sl}
\SetMathAlphabet{\mathsfit}{bold}{\encodingdefault}{\sfdefault}{bx}{n}

\def\gA{{\mathcal{A}}}

\def\gC{{\mathcal{C}}}
\def\gD{{\mathcal{D}}}

\def\gH{{\mathcal{H}}}

\def\gL{{\mathcal{L}}}
\def\gM{{\mathcal{M}}}

\def\gO{{\mathcal{O}}}
\def\gP{{\mathcal{P}}}

\def\gS{{\mathcal{S}}}

\def\gX{{\mathcal{X}}}
\def\gY{{\mathcal{Y}}}

\newcommand{\E}{\mathbb{E}}

\newcommand{\R}{\mathbb{R}}

\DeclareMathOperator*{\argmax}{arg\,max}

\usepackage{mathtools}
\usepackage{algorithm}%
\usepackage[noend]{algpseudocode}
\usepackage{wrapfig}
\usepackage{thmtools}
\usepackage{thm-restate}
\usepackage{enumitem}

\newcommand{\ctrain}{\gC_{\text{train}}}

\newcommand{\po}{\text{po}}
\DeclareMathOperator\supp{supp}

\graphicspath{{./figures/}}

\title{
Why Generalization in RL is Difficult: Epistemic POMDPs and Implicit Partial Observability
}

\author{
Dibya Ghosh\thanks{Contributed equally. Correspond to: \texttt{dibya@berkeley.edu},  \texttt{jrahme@math.princeton.edu}}
\\UC Berkeley \And Jad Rahme$^*$ \\ Princeton University \And Aviral Kumar \\ UC Berkeley \And Amy Zhang \\ UC Berkeley\\Facebook AI Research \AND Ryan P. Adams \\ Princeton University \And Sergey Levine \\ UC Berkeley
}

\begin{document}

\maketitle
\begin{abstract}
Generalization is a central challenge for the deployment of reinforcement learning (RL) systems in the real world. In this paper, we show that the sequential structure of the RL problem necessitates new approaches to generalization beyond the well-studied techniques used in supervised learning. While supervised learning methods can generalize effectively without explicitly accounting for epistemic uncertainty, we show that, perhaps surprisingly, this is not the case in RL. We show that generalization to unseen test conditions from a limited number of training conditions induces implicit partial observability, effectively turning even fully-observed MDPs into POMDPs. Informed by this observation, we recast the problem of generalization in RL as solving the induced partially observed Markov decision process, which we call the epistemic POMDP. We demonstrate the failure modes of algorithms that do not appropriately handle this partial observability, and suggest a simple ensemble-based technique for approximately solving the partially observed problem. Empirically, we demonstrate that our simple algorithm derived from the epistemic POMDP achieves significant gains in generalization over current methods on the Procgen benchmark suite.

\end{abstract}

\section{Introduction}
Generalization is a central challenge in machine learning. However, much of the research on reinforcement learning (RL) has been concerned with the problem of optimization: how to master a specific task through online or logged interaction.  Generalization to new test-time contexts
has received comparatively less attention, although
several works have observed empirically  \citep{Farebrother2018GeneralizationAR, Zhang2018ASO, Justesen2018IlluminatingGI, Song2020ObservationalOI} that generalization to new situations poses a significant challenge to RL policies learned from a fixed training set of situations. 
In standard supervised learning, it is known that in the absence of distribution shift and with appropriate inductive biases, optimizing for performance on the training set (i.e., empirical risk minimization) translates into
good generalization performance. 
It is tempting to suppose that the generalization challenges in RL can be solved in the same manner as empirical risk minimization in supervised learning: when provided a training set of contexts,
learn the optimal policy within these contexts and then use that policy in new contexts at test-time.

Perhaps surprisingly, we show that such ``empirical risk minimization'' approaches can be sub-optimal for generalizing to new contexts in RL, even when these new contexts are drawn from the same distribution as the training contexts. As an anecdotal example of why this sub-optimality arises, imagine a robotic zookeeper for feeding otters that must be trained on some set of zoos. When placed in a new zoo, the robot must find and enter the otter enclosure. It can use one of two strategies: either peek through all the habitat windows looking for otters, which succeeds with 95\% probability in all zoos, or to follow an image of a hand-drawn map of the zoo that unambiguously identifies the otter enclosure, which will succeed as long as the agent is able to successfully parse the image. In every training zoo, the otters can be found more reliably using the image of the map, and so an agent trained to seek the optimal policy in the training zoos would learn a classifier to predict the identity of the otter enclosure from the map, and enter the predicted enclosure. This classification strategy is optimal on the training environments because the agent can learn to perfectly classify the training zoo maps, but it is \textit{sub-optimal} for generalization, because the learned classifier will never be able to perfectly classify every new zoo map at test-time.
Note that this task is \emph{not} partially observed, because the map provides full state information even for a memoryless policy.
However, if the learned map classifier succeeds on anything less than 95\% of new zoos at test-time, the strategy of peeking through the windows, although always sub-optimal in the training environments, turns out to be a more reliable strategy for finding the otter habitat in a \textit{new} zoo, and results in higher expected returns at test-time. %

\begin{figure}
    \centering
    \includegraphics[width=\linewidth]{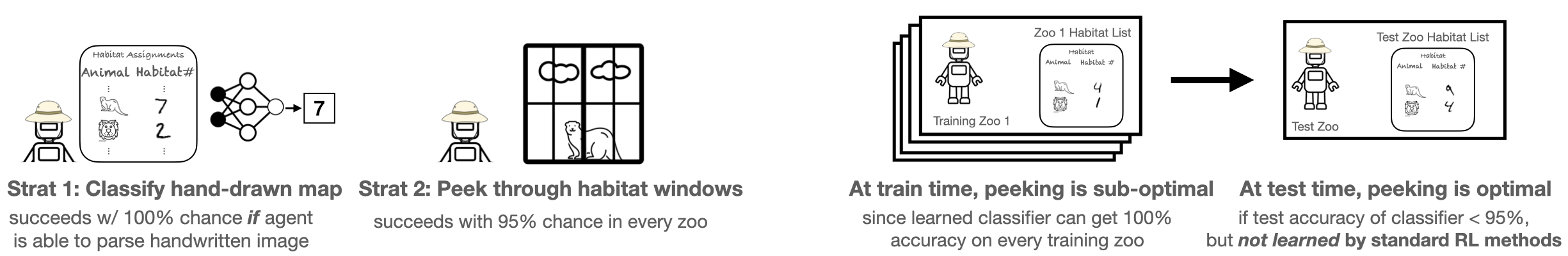}
        \vspace{-1em}
    \caption{\footnotesize{\textbf{Visualization of the robotic zookeeper example.} Standard RL algorithms learn the classifier strategy, since it is optimal in every training zoo, but this strategy is sub-optimal for generalization because peeking generalizes better than the classifier at test-time.  This failure occurs due to the following disconnect: while the task is \textit{fully-observed} since the image uniquely specifies the location of the otter habitat, to an agent that has limited training data, the location is \textit{implicitly partially observed at test-time} because of the agent's epistemic uncertainty about the parameters of the image classifier.}}
    \label{fig:my_label}
    \vspace{-2em}
\end{figure}

Although with enough training zoos, the zookeeper can learn a policy by solving the map classification problem, to generalize optimally when given a limited number of zoos requires a more intricate policy that is not learned by standard RL methods. %
How can we more generally describe the set of behaviors needed for a policy to generalize from a finite number of training contexts in the RL setting? We make the observation that, even in fully-observable domains, the agent's epistemic uncertainty renders the environment \textit{implicitly} partially observed at test-time. In the zookeeper example, although the hand-drawn map provides the exact location of the otter enclosure (and so the enclosure's location is technically fully observed), the agent cannot identify the true parameters of the map classifier from the small set of maps seen at training time, and so the location of the otters is implicitly obfuscated from the agent. %
We formalize this observation, and show that generalizing optimally at test-time corresponds to solving a partially-observed Markov decision process that we call an \textbf{epistemic POMDP}, which is induced by the agent's epistemic uncertainty about the test environment. 

That uncertainty about MDP parameters can be modelled as a POMDP is well-studied in Bayesian RL when training and testing on a single task in an online setting, primarily in the context of exploration \citep{Dearden1998BayesianQ, Duff2002OptimalLC, Strens2000ABF, Ghavamzadeh2015BayesianRL}. However, as we will discuss, this POMDP interpretation has significant consequences for the generalization problem in RL, where an agent cannot collect more data online, and must instead learn a policy from a fixed set of training contexts that generalizes to new contexts at test-time. We show that standard RL methods that do not explicitly account for this implicit partial observability can be arbitrarily sub-optimal for test-time generalization in theory and in practice. The epistemic POMDP underscores the difficulty of the generalization problem in RL, as compared to supervised learning, and provides an avenue for understanding how we should approach generalization under the sequential nature and non-uniform reward structure of the RL setting. %
Maximizing expected return in an approximation of the epistemic POMDP emerges as a principled approach to learning policies that generalize well, and we propose LEEP, an algorithm that uses an ensemble of policies to approximately learn the Bayes-optimal policy for maximizing test-time performance. %

The primary contribution of this paper is to use Bayesian RL techniques to reframe generalization in RL as the problem of solving a partially observed Markov decision process, which we call the \textit{epistemic POMDP}. The epistemic POMDP highlights the difficulty of generalizing well in RL, as compared to supervised learning. We demonstrate the practical failure modes of standard RL methods, which do not reason about this partial observability, and show that maximizing test-time performance may require algorithms to explicitly consider the agent's epistemic uncertainty during training. Our work highlights the importance of not only finding ways to help neural networks in RL generalize better, but also on learning policies that degrade gracefully when the underlying neural network eventually does fail to generalize. Empirically, we demonstrate that LEEP, which maximizes return in an approximation to the epistemic POMDP, achieves significant gains in test-time performance over standard RL methods on several ProcGen benchmark tasks.

\vspace{-1em}
\section{Related Work}
\vspace{-0.5em}
Many empirical studies have demonstrated the tendency of RL algorithms to overfit significantly to their training environments \citep{Farebrother2018GeneralizationAR, Zhang2018ASO, Justesen2018IlluminatingGI, Song2020ObservationalOI}, and the more general increased difficulty of learning policies that generalize in RL as compared to seemingly similar supervised learning problems \citep{Zhang2018NaturalEB, Zhang2018ADO, Whiteson2011ProtectingAE, Liu2020RegularizationMI}. These empirical observations have led to a newfound interest in algorithms for generalization in RL, and the development of benchmark RL
environments that focus on generalization to new contexts from a limited set of
training contexts sharing a similar structure (state and action spaces) but possibly different dynamics and rewards~\citep{Nichol2018GottaLF, Cobbe2019QuantifyingGI, Kuttler2020TheNL, Cobbe2020LeveragingPG, Stone2021TheDC}.

\textbf{Generalization in RL.} Approaches for improving generalization in RL have fallen into two main categories: improving the ability of  function approximators to
generalize better with inductive biases, and incentivizing behaviors that are easier to generalize to unseen contexts. To improve the representations learned in RL, prior work has considered imitating environment dynamics \citep{Jaderberg2017ReinforcementLW, Stooke2020DecouplingRL},  seeking bisimulation relations \citep{Zhang2020LearningIR, Agarwal2021ContrastiveBS}, and more generally,  addressing representational challenges in the RL optimization process \citep{Igl2019GeneralizationIR, Jiang2020PrioritizedLR}. In image-based domains,  inductive biases imposed via neural network design have also been proposed to improve robustness to certain factors of variation in the state~\citep{Lee2020NetworkRA,Kostrikov2020ImageAI, Raileanu2020AutomaticDA}. The challenges with generalization in RL that we will describe in this paper stem from the deficiencies of MDP objectives, and cannot be fully solved by choice of representations or functional inductive biases.
In the latter category, one approach is domain randomization, varying environment parameters such as coefficients of friction or textures, to obtain behaviors that are effective across many candidate parameter settings~\citep{Sadeghi2017CAD2RLRS, Tobin2017DomainRF, Rajeswaran2017EPOptLR, Sim2Real2018, kang2019generalization}.
Domain randomization sits within a class of methods that seek robust policies by injecting noise into the agent-environment loop, whether in the state \citep{Stulp2011LearningTG}, the action (e.g., via max-entropy RL) \citep{Cobbe2019QuantifyingGI}, or intermediary layers of a neural network policy  (e.g., through information bottlenecks) \citep{Igl2019GeneralizationIR, Lu2020DynamicsGV}. In doing so, these methods effectively introduce partial observability into the problem; while not necessarily equivalent to the partial observability that is induced by the epistemic POMDP, it may indicate why these methods generalize well empirically.

\textbf{Bayesian RL:} Our work recasts generalization in RL within the Bayesian RL framework, the problem of acting optimally %
under a belief distribution over MDPs (see Ghavamzadeh et al.~\citep{Ghavamzadeh2015BayesianRL} for a survey). Bayesian uncertainty has been studied in many sub-fields of RL \citep{Ramachandran2007BayesianIR, Lazaric2010BayesianMR, Jeon2018ABA, Zintgraf2020VariBADAV}, the most prominent being for exploration and learning efficiently in the online RL setting. Bayes-optimal behavior in RL is often reduced to acting optimally in a POMDP, or equivalently, a belief-state MDP \citep{Duff2002OptimalLC}, of which our epistemic POMDP is a specific instantiation. Learning the Bayes-optimal policy exactly is intractable in all but the simplest problems \citep{Weber1992OnTG, Poupart2006AnAS}, and many works in Bayesian RL have studied relaxations that remain asymptotically optimal for learning, for example with value of perfect information \citep{Dearden1998BayesianQ, Dearden1999ModelBB} or Thompson sampling \citep{Strens2000ABF, Osband2013MoreER, Russo2014LearningTO}.
Our main contribution is to revisit these classic ideas in the context of generalization for RL. We find that the POMDP interpretation of Bayesian RL \citep{Dearden1998BayesianQ, Duff2002OptimalLC, ross2007bayes} provides new insights on inadequacies of current algorithms used in practice, and explains why generalization in RL can be more challenging than in supervised learning. Being Bayesian in the generalization setting also requires new tools and algorithms beyond those classically studied in Bayesian RL, since test-time generalization is measured using regret over a \textit{single} evaluation episode, instead of throughout an online training process. As a result, algorithms and policies that minimize short-term regret (i.e., are more exploitative) are preferred over traditional algorithms like Thompson sampling that explore thoroughly to ensure asymptotic optimality at the cost of short-term regret.

\section{Problem Setup}
We consider the problem of learning RL policies given a set of training contexts that generalize well to new unseen contexts. This problem can be formalized in a Markov decision process (MDP) where the agent does not have full access to the MDP at training time, but only particular initial states or conditions. Before we describe what this means, we must describe the MDP~$\gM$, which is given by a tuple
$(\gS, \gA, r, T, \rho, \gamma)$, with state space $\gS$, action space $\gA$, Markovian transition function~$T(s_{t+1} |s_t, a_t)$, bounded reward function $r(s_t,a_t)$, and initial state distribution $\rho(s_0)$. A policy $\pi$ induces a discounted state distribution $d^{\pi}(s) = (1-\gamma) \mathbb{E}_{\pi}[\sum_{t \geq 0} \gamma^t 1(s_t = s)]$, and achieves return ~${J_\gM(\pi) = \E_\pi[\sum_{t \geq 0} \gamma^tr(s_t, a_t)]}$ in the MDP.
Classical results establish that a deterministic Markovian policy~$\pi^*$ maximizes this objective amongst all  history-dependent policies.

We focus on generalization in contextual MDPs where the agent is only trained on a subsample of contexts, and seeks to generalize well across unseen contexts. A contextual MDP is an MDP in which the state can be decomposed as~${s_t = (c, s_t')}$, a context vector~${c \in \gC}$ that remains constant throughout an episode, and a sub-state~${s' \in \gS'}$ that may vary:~${\gS \coloneqq \gC \times \gS'}$. Each context vector corresponds to a different situation that the agent might be in, each with slightly different dynamics and rewards, but some shared structure across which an agent can generalize.  During training, the agent is allowed to interact only within a sampled subset of contexts~${\ctrain \subset \gC}$. The generalization performance of the agent is measured by the return of the agent's policy in the full contextual MDP $J(\pi)$, corresponding to expected performance when placed in potentially new and unseen contexts. While our examples and experiments will be in contextual MDPs, our theoretical results also apply to other RL generalization settings where the full MDP cannot be inferred unambiguously from the data available during training, for example in offline reinforcement learning \citep{levine2020offline}.

\section{Warmup: A Sequential Classification RL Problem}
\label{sec:classification_as_rl}
\begin{figure}
    \centering
    \includegraphics[width=\linewidth]{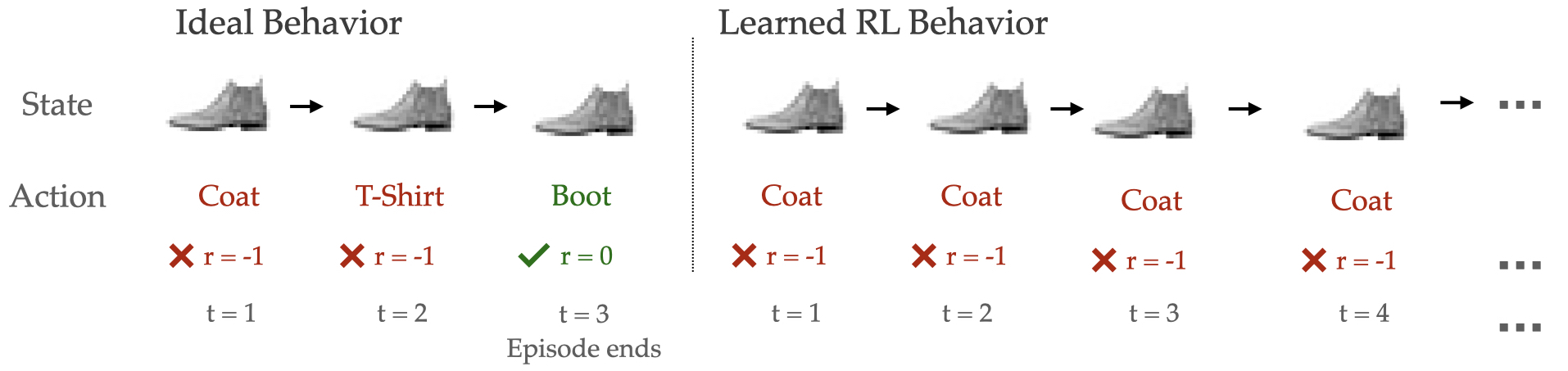}
    \vspace{-1.5em}
    \caption{\footnotesize{\textbf{ Sequential Classification RL Problem.} In this task, an agent must keep guessing the label for an image until it gets it correct. To avoid low test return, policies should change actions if the label guessed was incorrect, but  standard RL methods fail to do so, instead guessing the same incorrect label repeatedly.}
    }
    \label{fig:classification}
    \vspace{-1.5em}
\end{figure}

\begin{wrapfigure}{r}{0.4\textwidth}
    \centering
    \includegraphics[width=0.95\linewidth]{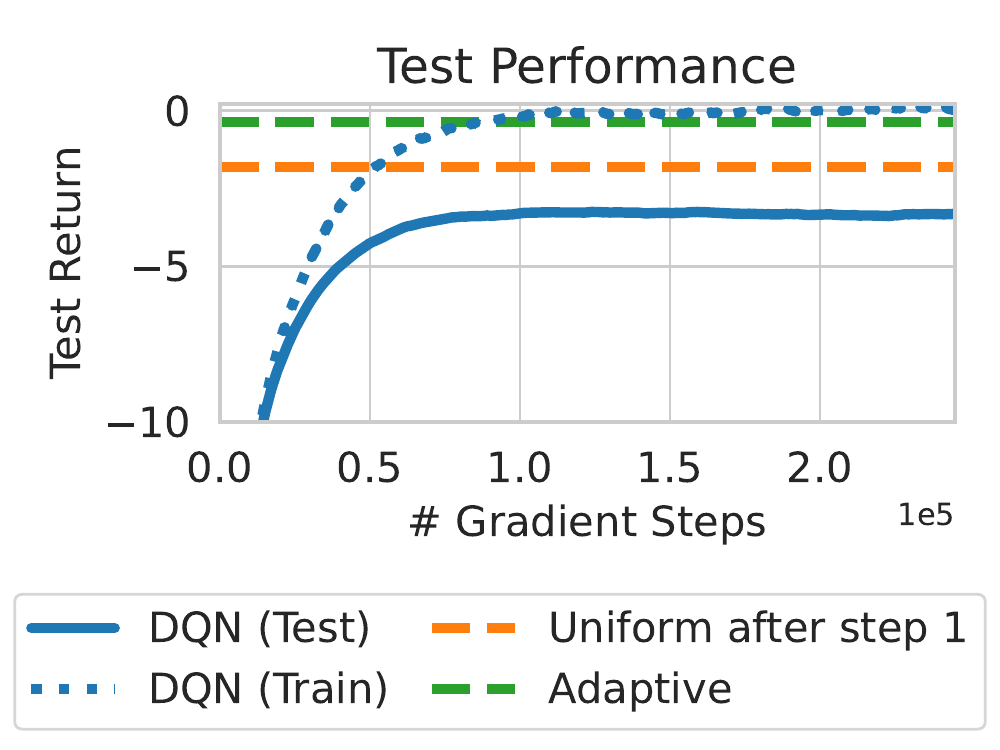}
    \includegraphics[width=0.95\linewidth]{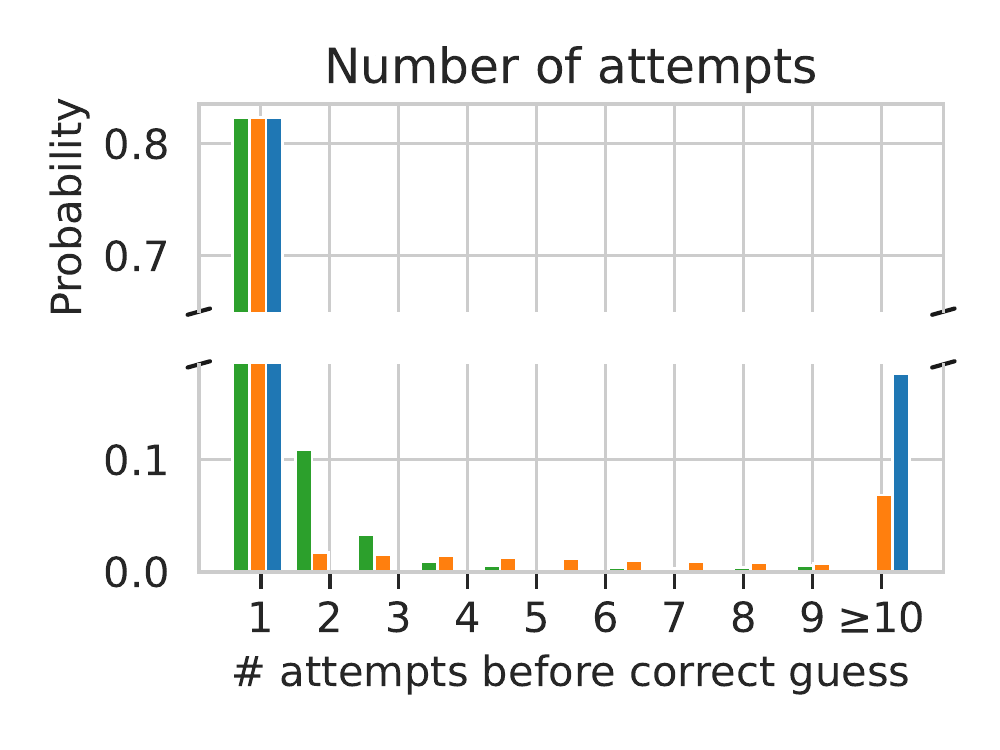}
    \caption{\footnotesize{\textbf{DQN on RL FashionMNIST.} DQN achieves lower test performance than simple variants that leverage the  structure of the RL problem.}}
\label{fig:classification_results}
\end{wrapfigure}We begin our study of generalization in RL with an example problem that is set up to be close to a supervised learning task where generalization is relatively well understood: image classification on the FashionMNIST dataset~\citep{FMNIST}. In this environment (visualized in Figure \ref{fig:classification}), an image from the dataset is sampled (the context) at the beginning of an episode and held fixed; the agent must identify the label of the image to complete the episode. If the agent guesses correctly, it receives a reward of $0$ and the episode ends; if incorrect, it receives a reward of $-1$ and the episode continues, so it must attempt another guess for the \emph{same} image at the next time step. This RL problem is near identical to supervised classification, the core distinction being that an agent may interact with the same image over several timesteps in an episode instead of only one attempt as in supervised learning. Note that since episodes may last longer than a single timestep, this problem is not a contextual bandit. %

The optimal policy in both the one-step and sequential RL version of the problem deterministically outputs the correct label for the image, because the image fully determines the label (in other words, it is a fully observed MDP). However, this optimal strategy generally cannot be learned from a finite training set, since some generalization error is unavoidable. With a fixed training set, the strategy for generalizing in classification remains the same: deterministically choose the label the agent is most confident about. However, the RL setting introduces two new factors: the agent gets multiple tries at classifying the same image, and it knows if an attempted label is incorrect. To generalize best to new test images, an RL policy must leverage this additional structure, for example by trying many possible labels, or by changing actions if the previous guess was incorrect.

A standard RL algorithm, which estimates the optimal policy on the empirical MDP defined by the dataset of training images %
does not learn to leverage these factors, and instead learns behavior highly sub-optimal for generalization. We obtained a policy by running DQN~\citep{DQN} (experimental details in Appendix~\ref{appendix:FashionMnistImplementation}), whose policy deterministically chooses the same label for the image at every timestep. Determinism is not specific to DQN, and is inevitable in any RL method that models the problem as an MDP because the optimal policy in the MDP is always deterministic and Markovian. The learned deterministic policy either guesses the correct label immediately, or guesses incorrectly and proceeds to make the same incorrect guess on every subsequent time-step. We compare performance in Figure \ref{fig:classification_results} with a version of the agent that starts to guess randomly if incorrect on the first timestep, and a different agent that acts by process of elimination: first choosing the action it is most confident about, if incorrect, then the second, and so forth. Although all three versions have the same training performance, the learned RL policy generalizes more poorly than these alternative variants that exploit the sequential nature of the problem. In Section~\ref{sec:understandingOptimality}, we will see that this process-of-elimination is, in some sense, the optimal way to generalize for this task. This experiment reveals a tension: learning policies for generalization that rely on an MDP model fail, even though the underlying environment \textit{is} an MDP. This failure holds in any MDP model with limited data, whether the empirical MDP or more sophisticated MDPs that use uncertainty estimates in their construction.

\section{Modeling Generalization in RL as an Epistemic POMDP}
To better understand test-time generalization in RL, we study the problem under a Bayesian perspective. We show that training on limited training contexts leads to an implicit partial observability at test-time that we describe using a formalism called the epistemic POMDP. %

\subsection{The Epistemic POMDP}
\label{sec:epistemic_pomdp}

In the Bayesian framework, when learning given a limited amount of evidence $\gD$ from an MDP $\gM$, we can use a prior distribution $\gP(\gM)$ to construct a posterior belief distribution $\gP(\gM | \gD)$ over the identity of the MDP. For learning in a contextual MDP, $\gD$ corresponds to the environment dynamics and reward in training contexts $\ctrain$ that the agent can interact with, and the posterior belief distribution $\gP(\gM | \gD)$ models the agent's uncertainty about the behavior of the environment in contexts that it has not seen before (e.g. uncertainty about the label for a image from the test set in the RL image classification example).

Since the agent only has partial access to the MDP $\gM$ during training, the agent does not know which MDP from the posterior distribution is the true environment, and must act at test-time under this uncertainty. Following a reduction common in Bayesian RL \citep{Duff2002OptimalLC, Ghavamzadeh2015BayesianRL}, we model this test-time uncertainty using a partially observed MDP that we will call the \textbf{epistemic POMDP}. The epistemic POMDP is structured as follows: each new episode in the POMDP begins by sampling a single MDP~${\gM \sim \gP(\gM | \gD)}$ from the posterior, and then the agent interacts with~$\gM$ until the episode ends in this MDP. The agent does not observe \textit{which} MDP was sampled, and since the MDP remains fixed for the duration of the episode, this induces implicit partial observability. 
Effectively, each episode in the epistemic POMDP corresponds to acting in one of the possible environments that is consistent with the evidence that the agent is allowed access to at training time.%

The epistemic POMDP is formally defined as the tuple~${\gM^{\po} = (\gS^\po, \gO^\po, \gA, T^\po, r^\po, \rho^\po, \gamma)}$.  A state in this POMDP~${s_t^\po = (\gM, s_t)}$ contains the identity of the current MDP being acted in~$\gM$, and the current state in this MDP~$s_t$; we write the state space as~${\gS^\po = \mathbf{M} \times \gS}$, where~$\mathbf{M}$ is the space of MDPs with support under the prior.
The agent only gets to observe~${o_t^\po = s_t}$,  the state in the MDP (${\gO^\po = \gS}$). The initial state distribution is defined by the posterior distribution over MDPs: $\rho^\po((\gM, s_0)) = \gP(\gM | \gD)\rho_\gM(s_0)$,
and the transition and reward functions in the POMDP reflect the dynamics in the current MDP:
\begin{equation}
    T^\po((\gM', s') \mid (\gM, s), a) = \delta(\gM' = \gM)T_{\gM}(s' | s,a)~~~~~~r^\po((\gM, s), a) = r_{\gM}(s,a)\,.
\end{equation}
What makes the epistemic POMDP a useful tool for understanding generalization in RL is that when the prior is well-specified, performance in the epistemic POMDP $\gM^\po$ corresponds exactly to the expected return of the agent at test-time.

\begin{restatable}{proposition}{epistemicpomdpdefn}
\label{prop:epistemic_pomdp_defn}
If the true MDP $\gM$ is sampled from $\gP(\gM)$, and evidence $\gD$ from $\gM$ is provided to an algorithm during training, then the expected test-time return of $\pi$ is equal to its performance in the epistemic POMDP $\gM^\po$.
\begin{equation}
\label{eqn:pomdp_objective}
    J_{\gM^\po}(\pi) %
    = \E_{\gM \sim \gP(\gM)}[J_{\gM}(\pi) \mid \gD].
\end{equation}
In particular, the optimal policy in $\gM^\po$ is Bayes-optimal for generalization to the unknown MDP $\gM$: it receives the highest expected test-time return amongst all possible policies.
\end{restatable}

The epistemic POMDP is based on well-understood concepts in Bayesian reinforcement learning, and Bayesian modeling more generally. However, in contrast to prior works on Bayesian RL, we are specifically concerned with settings where there is a training-test split, and performance is measured by a single test episode. While using Bayesian RL to accelerate exploration or minimize regret has been well-explored~\citep{Ghavamzadeh2015BayesianRL}, we rather use the Bayesian lens specifically to understand generalization -- a perspective that is distinct from prior work on Bayesian RL. Towards this goal, the equivalence between test-time return and expected return in the epistemic POMDP allows us to use performance in the POMDP as a proxy for understanding how well current RL methods can generalize.

\subsection{Understanding Optimality in the Epistemic POMDP}
\vspace{-0.5em}
\label{sec:understandingOptimality}

We now study the structure of the epistemic POMDP, and use it to characterize properties of Bayes-optimal test-time behavior and the sub-optimality of alternative policy learning approaches. The majority of our results follow from well-known results about POMDPs, so we present them here informally, with formal statements and proofs in Appendix \ref{appendix:Sec5}. We begin by revisiting the sequential image classification problem, and describing the induced epistemic POMDP. 

\textbf{Example }(Sequential Image Classification)\textbf{.} ~~In the task from Section \ref{sec:classification_as_rl}, the agent's uncertainty concerns  how images are mapped to labels. Each MDP $\gM$ in the posterior distribution corresponds to a different potential classifier $Y_\gM$, where $Y_\gM(x)$ denotes the label of $x$ according to $\gM$. Each episode in the epistemic POMDP requires the agent to guess the label for an image, where the label is determined by a randomly chosen classifier whose identity is unknown. Acting optimally in the epistemic POMDP requires strategies that work well in expectation over the distribution of labels that is induced by the posterior distribution $p(y|x, \gD) = \E_{\gM \sim \gP(\gM|\gD)}[1(Y_\gM(x) = y)]$. A deterministic policy (as is learned by standard RL algorithms) is a high-risk strategy in the POMDP; it receives decent return when the guessed label matches the one from the sampled classifier, but exceedingly low return if the sampled classifier outputs a different label (-1 at every timestep). The Bayes-optimal strategy for maximizing expected test-time return corresponds to a process of elimination: first choose the most likely label $a = \argmax p(y|x, \gD)$; if this is incorrect, eliminate it and choose the next-most likely, repeating until the correct label is finally chosen. Amongst memoryless policies, the optimal behavior is stochastic, sampling actions according to the distribution $\pi^*(a|x) \propto \sqrt{p(y|x, \gD)}$ (derivation in Appendix~\ref{appendix:FashionMnistTheory}).

The characteristics of the optimal policy in the epistemic POMDP for the image classification RL problem match well-established results that optimal POMDP policies are generally memory-based \citep{monahan1982state}, and amongst memoryless policies, the optimal policy may be stochastic \citep{singhPOMDP, Montfar2015GeometryAD}.
Because of the equivalence between the epistemic POMDP and test-time behavior, these maxims are also true for Bayes-optimal behavior when maximizing test-time performance.
\begin{principle}
The Bayes-optimal policy for maximizing test-time performance is in general non-Markovian. When restricted to Markovian policies, the Bayes-optimal policy is in general stochastic.
\end{principle}
The reason that Bayes-optimal behavior is often non-Markovian is that the experience collected thus far in the episode contains information about the identity of the MDP being acted in (which is hidden from the agent observation), and to maximize expected return, the agent must adapt its subsequent behavior to incorporate this new information. The fact that acting optimally at test-time formally requires adaptivity (or stochasticity for memoryless policies) highlights the difficulty of generalizing well in RL, and provides a new perspective for understanding the success various empirical studies have found in  improving generalization performance using recurrent networks \citep{openAIcube, simToRealPeng} and stochastic regularization penalties \citep{Stulp2011LearningTG, Cobbe2019QuantifyingGI, Igl2019GeneralizationIR, Lu2020DynamicsGV}.

It is useful to understand to what degree the partial observability plays a role in determining Bayes-optimal behavior. When the partial observability is insignificant, the epistemic POMDP objective can coincide with a surrogate MDP approximation, and Bayes-optimal solutions can be attained with standard fully-observed RL algorithms. For example, if there is a policy that is simultaneously optimal in \textit{every} MDP from the posterior, then an agent need not worry about the (hidden) identity of the MDP, and just follow this policy. Perhaps surprisingly, this kind of condition is difficult to relax: we show in Proposition \ref{prop:deterministic-dominated-random} that even if a policy is optimal in many (but not all) of the MDPs from the posterior, this seemingly ``optimal'' policy can generalize poorly at test-time.

Moreover, under partial observability, optimal policies for the MDPs in the posterior may differ substantially from Bayes-optimal behavior: in Proposition \ref{prop:suboptimal-actions-everywhere}, we show that the Bayes-optimal policy may take actions that are sub-optimal in \textit{every} environment in the posterior. These results indicate the brittleness of learning policies based on optimizing return in an MDP model when the agent has not yet fully resolved the uncertainty about the true MDP parameters.
\begin{principle}[Failure of MDP-Optimal Policies, Propositions \ref{prop:deterministic-dominated-random}, \ref{prop:suboptimal-actions-everywhere}]
The expected test-time return of policies that are learned by maximizing reward in any MDP from the posterior, as standard RL methods do, may be arbitrarily low compared to that of Bayes-optimal behavior. 
\end{principle}

As Bayes-optimal memoryless policies are stochastic, one may wonder if simple strategies for inducing stochasticity, such as adding $\epsilon$-greedy noise or entropy regularization, can alleviate the sub-optimality that arose with deterministic policies in the previous paragraph. 
In some cases, this may be true; one particularly interesting result is that in certain goal-reaching problems, entropy-regularized RL can be interpreted as optimizing an epistemic POMDP objective with a specific form of posterior distribution over reward functions (Proposition~\ref{appendix:MaxEntRL}) \citep{eysenbachMaxEnt}. For the more general setting, we show in Proposition \ref{prop:stochastic-dominated-random} that entropy regularization and other general-purpose techniques can similarly catastrophically fail in epistemic POMDPs.

\begin{principle}[Failure of Generic Stochasticity, Proposition \ref{prop:stochastic-dominated-random}]
The expected test-time return of policies learned with stochastic regularization techniques like maximum-entropy RL that are agnostic of the posterior $\gP(\gM|\gD)$ may be arbitrarily low compared to that of Bayes-optimal behavior.
\end{principle}

This failure happens because the degree of stochasticity used by the Bayes-optimal policy reflects the agent's epistemic uncertainty about the environment; since standard regularizations are agnostic to this uncertainty, the learned behaviors often do not reflect the appropriate level of stochasticity needed. A maze-solving agent acting Bayes-optimally, for example, may choose to act deterministically in mazes like those it has seen at training, and on others where it is less confident, rely on random exploration to exit the maze, inimitable behavior by regularization techniques agnostic to this uncertainty.

Our analysis of the epistemic POMDP highlights the difficulty of generalizing well in RL, in both the complexity of Bayes-optimal policies (Remark 5.1) and the deficiencies of our standard MDP-based RL algorithms (Remark 5.2, 5.3) . While MDP-based algorithms can serve as a useful starting point for acquiring generalizable skills, learning policies that perform well in new test-time scenarios may require more complex algorithms that attend to the epistemic POMDP structure that is implicitly induced by the by the agent's epistemic uncertainty.

\vspace{-0.5em}
\section{Learning Policies that Generalize Well Using the Epistemic POMDP}
\vspace{-0.5em}
When the epistemic POMDP $\gM^\po$ can be exactly obtained, we can learn RL policies that generalize well to the true (unknown) MDP $\gM$ by learning an optimal policy in the POMDP. In this oracle setting, any POMDP-solving method will suffice, and design choices like policy function classes (e.g. recurrent vs Markovian policies) or agent representations (e.g. belief state vs PSRs) made based on the requirements of the specific domain. However, in practice, the epistemic POMDP can be challenging to approximate due to the difficulties of learning coherent MDP models and maintaining a posterior over such MDP models in high-dimensional domains. 

In light of these challenges, we now focus on practical methods for learning generalizable policies when the exact posterior distribution (and therefore true epistemic POMDP) cannot be recovered exactly. We derive an algorithm for learning the optimal policy in the epistemic POMDP induced by an approximate posterior distribution $\hat{\gP}(\gM| \gD)$ with finite support. We use this to motivate LEEP, a simple ensemble-based algorithm for learning policies in the contextual MDP setting.

\vspace{-0.5em}
\subsection{Policy Optimization in an Empirical Epistemic POMDP}
\vspace{-0.5em}

Towards a tractable algorithm, we assume that instead of the true posterior $\gP(\gM | \gD)$, we only have access to an empirical posterior distribution $\hat{\gP}(\gM | \gD)$ defined by $n$ MDP samples from the posterior distribution $\{\gM_i\}_{i\in [n]}$. This empirical posterior distribution induces an empirical epistemic POMDP $\hat{\gM}^\po$; our ambition is to learn the optimal policy in this POMDP. Rather than directly learning this optimal policy as a generic POMDP solver might, we recognize that $\hat{\gM}^\po$ corresponds to a collection of $n$ MDPs \footnote{Note that when the true environment is a contextual MDP, the sampled MDP $\gM_i$ does not correspond to a single context within a contextual MDP --- each MDP $\gM_i$ is an \emph{entire} contextual MDP with many contexts.
} and decompose the optimization problem to mimic this structure. We will learn $n$ policies $\pi_1, \cdots, \pi_n$, each policy $\pi_i$ in one of the MDPs $\gM_i$ from the empirical posterior, and combine these policies together to recover a single policy $\pi$ for the POMDP. Reducing the POMDP policy learning problem into a set of MDP policy learning problems can allow us to leverage the many recent advances in deep RL for scalably solving MDPs. The following theorem links the expected return of a policy $\pi$ in the empirical epistemic POMDP $\hat{\gM^\po}$, in terms of the performance of the policies $\pi_i$ on their respective MDPs $\gM_i$, and provides a natural objective for learning policies in this decoupled manner.

\begin{restatable}{proposition}{theoremLowerBound}
\label{thm:lower_bound}
Let $\pi,\pi_1,\cdots \pi_n$ be memoryless policies, and define $r_{\max} = \max_{i, s, a} |r_{\gM_i}(s,a)|$.  The expected return of $\pi$ in $\hat{\gM}^\po$
is bounded below as:
\vspace{-0.5em}
\begin{equation}
\label{eqn:lower_bound}
         J_{\hat{\gM}^\po}(\pi) \geq \frac{1}{n} \sum_{i=1}^n J_{\gM_i}(\pi_i) - \frac{\sqrt{2}r_{\max}}{(1-\gamma)^2n} \sum_{i=1}^n \E_{s \sim d_{\gM_i}^{\pi_i}}\left[\sqrt{ D_{KL}\left(\pi_i(\cdot | s)\,\, || \,\, \pi(\cdot | s)\right)}\right],
\end{equation}
\vspace{-0.5em}
\end{restatable}
\vspace{-0.5em}

This proposition indicates that if the policies in the collection $\{\pi_i\}_{i \in [n]}$ all achieve high return in their respective MDPs (first term) and are imitable by a single policy $\pi$ (second term), then $\pi$ is guaranteed to achieve high return in the epistemic POMDP. In contrast, if the policies cannot be closely imitated by a single policy, this collection of policies may not be useful for learning in the epistemic POMDP using the lower bound. This means that it may not sufficient to naively optimize each policy $\pi_i$ on its MDP $\gM_i$ without any consideration to the other policies or MDPs, since the learned policies are likely to be different and difficult to jointly imitate. To be useful for the lower bound, each policy $\pi_i$ should balance between maximizing performance on its MDP and minimizing its deviation from the other policies in the set. The following proposition shows that if the policies are trained jointly to ensure this balance, it in fact recovers the optimal policy in the empirical epistemic POMDP.

\begin{restatable}{proposition}{propositionLinkFunction}
\label{prop:link-version}
Let $f: \{\pi_i\}_{i\in[n]} \mapsto \pi$ be a function that maps $n$ policies to a single policy satisfying $f(\pi,\cdots,\pi) = \pi $ for every policy $\pi$, and let $\alpha$ be a hyperparameter satisfying $\alpha \geq \frac{\sqrt{2} r_{\text{max}}}{(1-\gamma)^2n}$. Then letting $\pi_1^*, \dots \pi_n^*$ be the optimal solution to the following optimization problem:
\begin{equation}
\label{eq:algo_objective}
    \{\pi^{*}_i\}_{i\in[n]} = \argmax_{\pi_1, \cdots, \pi_n} \frac{1}{n} \sum_{i=1}^n J_{\gM_i}(\pi_i)  - \alpha \sum_{i=1}^n \E_{s \sim d_{\gM_i}^{\pi_i}}\left[\sqrt{ D_{KL}\left(\pi_i(\cdot | s)\,\, || \,\, f(\{\pi_i\})(\cdot | s)\right)}\right],
\end{equation}
the policy $\pi^{*} \coloneqq f(\{\pi_i^*\}_{i\in[n]})$ is optimal for the empirical epistemic POMDP $\hat{\gM}^\po$.
\end{restatable}

\subsection{A Practical Algorithm for Contextual MDPs: LEEP}

Proposition \ref{prop:link-version} provides a foundation for a practical algorithm for learning policies when provided training contexts $\ctrain$ from an unknown contextual MDP. In order to use the proposition in a practical algorithm, we must discuss two problems: how posterior samples $\gM_i \sim \gP(\gM | \gD)$ can be approximated, and how the function $f$ that combines policies should be chosen.

\textbf{Approximating the posterior distribution: } Rather than directly maintaining a posterior over transition dynamics and reward models, which is especially difficult with image-based observations, we can approximate samples from the posterior via a bootstrap sampling technique \citep{bootstrappedDQN}. To sample a candidate MDP $\gM_i$, we sample with replacement from
the training contexts $\ctrain$ to get a new set of contexts $\ctrain^i$, and define $\gM_i$ to be the empirical MDP on this subset of training contexts. Rolling out trials from the posterior sample $\gM_i$ then corresponds to selecting a context at random from $\ctrain^i$, and then rolling out that context. Crucially, note that $\gM_i$ still corresponds to a \emph{distribution} over contexts, not a single context, since our goal is to sample from the posterior entire contextual MDPs.
\begin{algorithm}[t] %
  \caption{Linked Ensembles for the Epistemic POMDP (LEEP)}\label{algobox}
  \begin{algorithmic}[1]
    \State Receive training contexts $\ctrain$, number of ensemble members $n$
    \State Bootstrap sample training contexts to create $\ctrain^1, \dots \ctrain^n$, where $\ctrain^i \subset \ctrain$.
\State Initialize $n$ policies: $\pi_1, \dots \pi_n$
    \For{iteration $k = 1, 2, 3, \dots $}
        \For{policy $i = 1, \dots,  n$}
      \State Collect environment samples in training contexts $\ctrain^i$ using policy $\pi_i$
      \State Take gradient steps wrt $\pi_i$ on these samples with augmented RL loss:
      $$\pi_i \gets \pi_i - \eta \nabla_i (\gL^{RL}(\pi_i) + \alpha \E_{s \sim \pi_i, \ctrain^i}[ D_{KL}(\pi_i(a|s) \| \max_j \pi_j(a|s))])$$
        \EndFor
    \EndFor
    \State Return $\pi = \max_i \pi$: $\pi(a|s) = \frac{\max_i \pi_i(a|s)}{\sum_{a'} \max_i \pi_i(a'|s)}$.
  \end{algorithmic}
\end{algorithm}

\textbf{Choosing a link function:}
The link function $f$ in Proposition \ref{prop:link-version} that combines the set of policies together effectively serves as an inductive bias: since we are optimizing in an approximation to the true epistemic POMDP and policy optimization is not exact in practice, different choices can yield combined policies with different characteristics. Since optimal behavior in the epistemic POMDP must consider all actions, even those that are potentially sub-optimal in all MDPs in the posterior (as discussed in Section~\ref{sec:understandingOptimality}), we use an ``optimistic'' link function that does not dismiss any action that is considered by at least one of the policies, specifically ${f(\{\pi_i\}_{i\in[n]}) = (\max_i \pi_i)(a|s) \coloneqq  \frac{\max \pi_i(a|s)}{\sum_{a'} \max \pi_i(a'|s)}}$. %

\textbf{Algorithm:} We learn a set of $n$ policies $\{\pi_{i}\}_{i \in [n]}$, using a policy gradient algorithm to implement the update step. To update the parameters for $\pi_{i}$, we take gradient steps via the surrogate loss used for the policy gradient, augmented by a disagreement penalty between the policy and the combined policy $f(\{\pi_{i}\}_{i \in [n]})$ with a penalty parameter $\alpha > 0$, as in Equation \ref{eq:ppo_surrogate}:
\begin{equation}
\label{eq:ppo_surrogate}
    \gL(\pi_i) = \gL^{RL}(\pi_i) + \alpha \E_{s \sim \pi_i, \gM_i}[ D_{KL}(\pi_i(a|s) \| \max_j \pi_j(a|s))].
\end{equation}
Combining these elements together leads to our method, LEEP, which we summarize in Algorithm \ref{algobox}. In our implementation, we use PPO for $\gL^{RL}(\pi_i)$ \citep{Schulman2017ProximalPO}. In summary, LEEP bootstrap samples the training contexts to create overlapping sets of training contexts $\ctrain^1, \dots \ctrain^n$. Every iteration, each policy $\pi_i$ generates rollouts in training contexts chosen uniformly from its corresponding $\ctrain^i$, and is then updated according to Equation \ref{eq:ppo_surrogate}, which both maximizes the expected reward and minimizes the disagreement penalty between each $\pi_i$ and the combined policy $\pi = \max_j \pi_j$.

While this algorithm is structurally similar to algorithms for multi-task learning that train a separate policy for each context or group of contexts with a disagreement penalty~\citep{Distral, ghosh2018divideandconquer}, the motivation and the interpretation of these approaches are completely distinct. In multi-task learning, the goal is to solve a given set of tasks, and these methods promote transfer via a loss that encourages the solutions to the tasks to be in agreement. In our setting, while we also receive a set of tasks (contexts), the goal is not to maximize performance on the training tasks, but rather to learn a policy that maximizes performance on unseen test tasks. The method also has a subtle but important distinction: each of our policies $\pi_i$ acts on a sample from the contextual MDP posterior, \emph{not} a single training context~\citep{Distral} or element from a disjoint partitioning~\citep{ghosh2018divideandconquer}. This distinction is crucial, since our aim is not to make it easier to solve the training contexts, but the opposite: prevent the algorithm from overfitting to the individual contexts. Correspondingly, our experiments confirm that such multi-task learning approaches do not provide the same generalization benefits as our approach.

\vspace{-0.5em}
\section{Experiments}
\vspace{-0.5em}
The primary ambition of our empirical study is to test the hypothesis that policies that are learned through (approximations of) the epistemic POMDP do in fact attain better test-time performance than those learned by standard RL algorithms. We do so on the Procgen benchmark~\citep{Cobbe2020LeveragingPG}, a challenging suite of diverse tasks with image-based observations testing generalization to unseen contexts. 
\vspace{3em}
\begin{enumerate}[topsep=0mm,itemsep=0.1mm]
    \item Does LEEP derived from the epistemic POMDP lead to improved test-time performance over standard RL methods?
    \item Can LEEP prevent overfitting when provided a limited number of training contexts?
    \item How do different algorithmic components of LEEP affect test-time performance ?
\end{enumerate}

The Procgen benchmark is a set of procedurally generated games, each with different generalization challenges. In each game, during training, the algorithm can interact with 200 training levels, before it is asked to generalize to the full distribution of levels. The agent receives a $64 \times 64 \times 3$ image observation, and must output one of 15 possible actions. We instantiate our method using an ensemble of $n=4$ policies, a penalty parameter of $\alpha=1$, and PPO \citep{Schulman2017ProximalPO} to train the individual policies (full implementation details in Appendix~\ref{appendix:LEEP}). 

\begin{figure}
    \centering
\includegraphics[width=\linewidth]{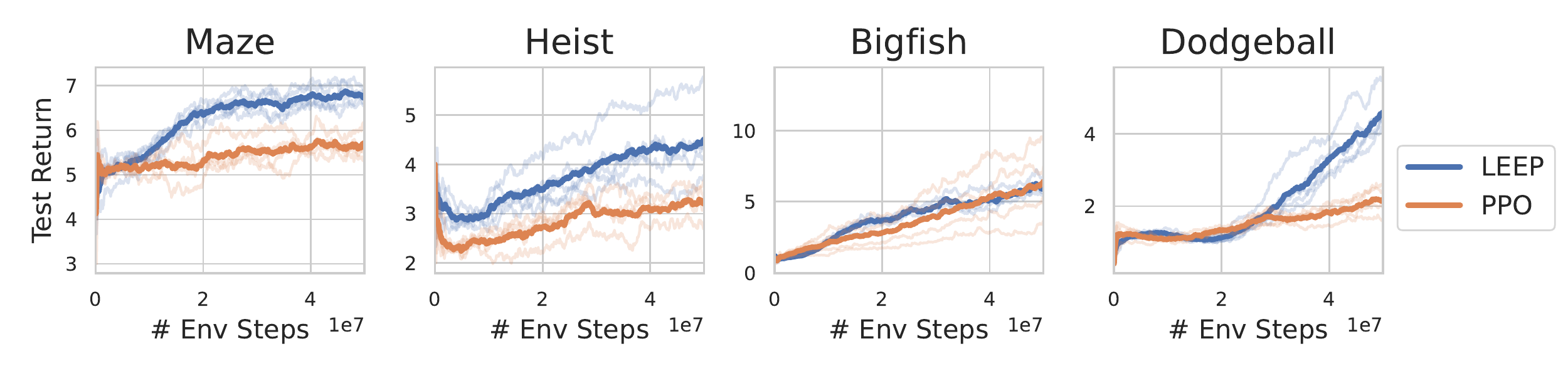}
\vspace{-0.5cm}
    \caption{Test set return for LEEP and PPO throughout training in four Procgen environments (averaged across 5 random seeds). LEEP achieves higher test returns than PPO on three tasks (Maze, Heist and Dodgeball) and matches test return on Bigfish while having less variance across seeds.   }
\vspace{-0.5cm}
    \label{fig:ppo-comparaison-2}
\end{figure}

We evaluate our method on four games in which prior work has found a large gap between training and test performance, and which we therefore conclude pose a significant generalization challenge~\citep{Cobbe2020LeveragingPG, Jiang2020PrioritizedLR, Raileanu2020AutomaticDA}:
Maze, Heist, BigFish, and Dodgeball. In Figure~\ref{fig:ppo-comparaison-2}, we compare the test-time performance of the policies learned using our method to those learned by a PPO agent with entropy regularization. In three of these environments (Maze, Heist, and Dodgeball), our method outperforms PPO by a significant margin, and in all cases, we find that the generalization gap between training and test performance is lower for our method than PPO (Appendix~\ref{appendix:morePlots}). %
\begin{wrapfigure}{r}{0.38\textwidth}
\vspace{-1em}
    \centering
    \includegraphics[width=\linewidth]{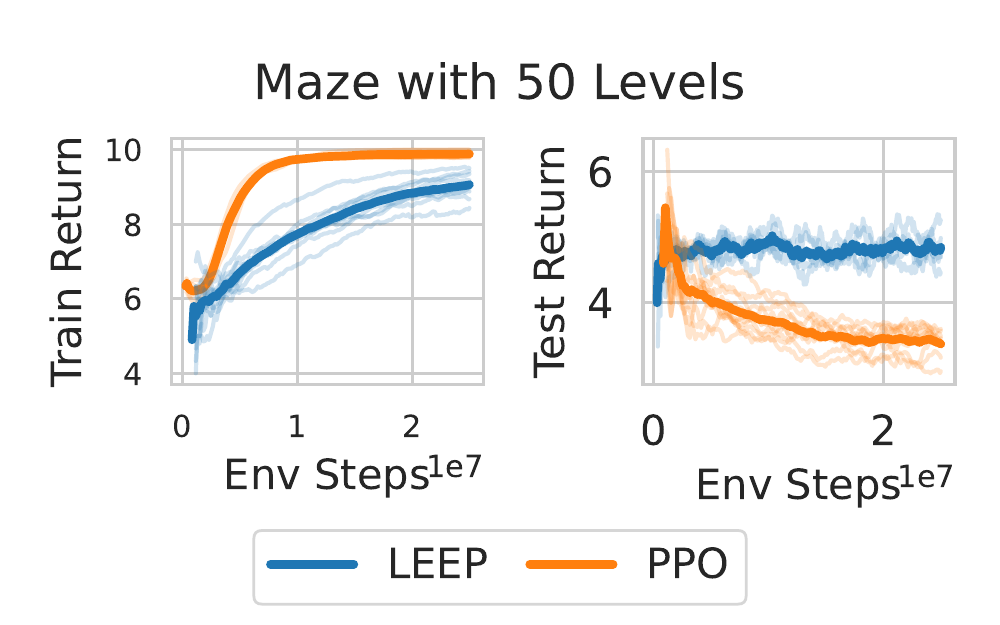}
    \includegraphics[width=\linewidth]{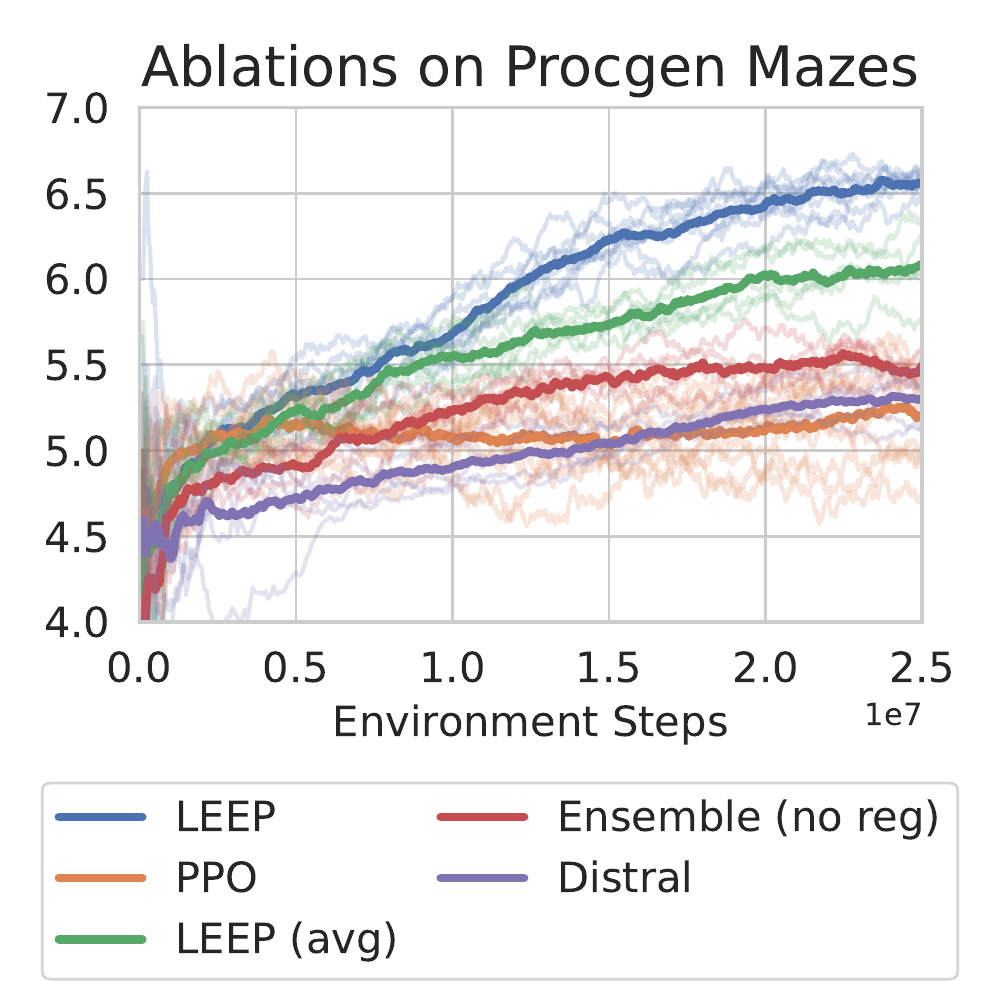}
 \vspace{-1em}   
    \caption{\footnotesize{\textbf{(top)} Performance of LEEP and PPO with only 50 training levels on Maze. \textbf{(bottom)}  Ablations of LEEP in Maze.}}
    \label{fig:ablations}
\end{wrapfigure}To understand how LEEP behaves with fewer training contexts, we ran on the Maze task with only 50 levels (Figure~\ref{fig:ablations} (top)); the test return of the PPO policy decreases through training, leading to final performance worse than the starting random policy, but our method avoids this degradation. %

We perform an ablation study on the Maze and Heist environments (Maze in Figure 4, Heist in Appendix~\ref{appendix:morePlots}) to rule out potential confounding causes for the improved generalization that our method displays on the Procgen benchmark tasks.
First, to see if the performance benefit derives solely from the use of ensembles, we compare LEEP to a Bayesian model averaging strategy that trains an ensemble of policies without regularization (``Ensemble (no reg)''), and uses a mixture of these policies. This strategy does improve performance over the PPO policy, but does not match LEEP, indicating the usefulness of the regularization. Second, we compared to a version of LEEP that combines the ensemble policies together using the average $\frac{1}{n}\sum_{i=1}^n \pi_i(a|s)$ (``LEEP (avg)'').
This link function achieves worse test-time performance than the optimistic version, which indicates that the inductive bias conferred by the $\max_i \pi_i$ link function is a useful component of the algorithm. Finally, we compare to Distral, a multi-task learning method with different motivations but similar structure to LEEP: this method helps accelerate learning on the provided training contexts (figures in Appendix~\ref{appendix:morePlots}), but does not improve generalization performance as LEEP does.

\vspace{-0.5em}
\section{Discussion}
\vspace{-0.5em}
It has often been observed experimentally that generalization in RL poses a significant challenge, but it has so far remained an open question as to whether the RL setting itself presents additional generalization challenges beyond those seen in supervised learning. In this paper, we answer this question in the affirmative, and show that, in contrast to supervised learning, generalization in RL results in a new type of problem that cannot be solved with standard MDP solution methods, due to partial observability induced by epistemic uncertainty. We call the resulting partially observed setting the epistemic POMDP, where uncertainty about the true underlying MDP results in a challenging partially observed problem. We present a practical approximate method that optimizes a bound for performance in an approximation of the epistemic POMDP, and show empirically that this approach, which we call LEEP, attains significant improvements in generalization over other RL methods that do not properly incorporate the agent's epistemic uncertainty into policy optimization. A limitation of this approach is that it optimizes a crude approximation to the epistemic POMDP with a small number of posterior samples, and may be challenging to scale to better approximations to the true objective. Developing algorithms that better model the epistemic POMDP and optimize policies within is an exciting avenue for future work, and we hope that this direction will lead to further improvements in generalization in RL.

\section*{Acknowledgements} This research was supported by an NSF graduate fellowship, the DARPA assured autonomy program, the NSF IIS-2007278 grant, a Princeton SEAS Innovation Grant and compute support from Google and Microsoft. We thank Benjamin Eysenbach, Xinyang Geng, and Justin Fu  as well as members of the Princeton Laboratory for Intelligent Probabilistic Systems for helpful discussions and feedback.

\bibliography{main}{}
\bibliographystyle{unsrt}
\clearpage
\appendix

\section{FashionMNIST Classification}
\label{appendix:FashionMnist}
\subsection{Implementation details}
\label{appendix:FashionMnistImplementation}

\textbf{Environment:} The RL image classification environment consists of a dataset of labelled images. At the beginning of each episode, a new image and its corresponding label are chosen from the dataset, and held fixed for the entire episode. Each time-step, the agent must pick an action corresponding to one of the labels. If the picked label is correct, the agent gets a reward of $r=0$, and the episode ends, and if the picked label is incorrect, then the agent gets a reward of $r=-1$, and the episode continues to the next time-step (where it must guess another label for the \textit{same} image). The total return for a trajectory corresponds to the number of incorrect guesses the agent makes for the image. We enforce a time-limit of $20$ timesteps in the environment to prevent infinite-length trajectories of incorrect guessing. %

We train the agent on a dataset of $10000$ FashionMNIST images subsampled from the training set, and test on the FashionMNIST test dataset. Note that this task is very similar, but not exactly equivalent to maximizing predictive accuracy for supervised classification: if the episode ended regardless of whether or not the agent was correct, then it would correspond exactly to classification. %

\textbf{Algorithm:} We train a DQN agent on the training environment using the min-Q update rule from TD3 \citep{Fujimoto2018AddressingFA}. The Q-function architecture is a convolutional neural network (CNN) with the architecture from Kostrikov et al \citep{pytorchrl}. To ensure that the agent does not suffer from poor exploration during training, the replay buffer is pre-populated with one copy of every possible transition in the training environment (that is, where every action is taken for every image in the training dataset). The variant labelled ``Uniform after step 1'' in Figure \ref{fig:classification_results} follows the DQN policy for the first time-step, and if this was incorrect, then at all subsequent time-steps, takes a random action uniformly amongst the 10 labels. For the variant labelled ``Adaptive'', we train a classifier $p_\theta(y|x)$ on the training dataset of images with the same architecture as the DQN agent. The adaptive agent follows a process-of-elimination strategy; formally, the action taken by the adaptive agent at time-step $t$ is given by $\argmax_{a \notin \{a_1, \dots, a_{t-1}\}} p_\theta(y=a|x)$.

\subsection{Derivation of Bayes-optimal policies}
\label{appendix:FashionMnistTheory}

In the epistemic POMDP for the RL image classification problem, each episode, an image $x \in \gX$ is sampled randomly from the dataset, and a label $y \in \gY$ sampled randomly for this image from the distribution $p(y|x, \gD)$. This label is \textit{held fixed} for the entire episode. For notation, let $Y = \{1, \dots, d\}$, so that a label distribution $p(y|x)$ can be written as a vector in the probability simplex on $\R^d$. We emphasize two settings: $\gamma=0$ (the supervised learning setting), and $\gamma=1$ (an RL setting), where the expected return of an agent is the average number of incorrect guesses made. 

\subsubsection{Memory-based policy}

Since the optimal memory-based policy in a POMDP is deterministic \citep{monahan1982state}, we restrict ourselves to analyzing the performance of deterministic memory-based policies. In the following we will narrow the search space even further.

Since the episode ends after the agent correctly classifies an image and the reward structure incentives the agent to solve the task as quickly as possible, an agent acting optimally will never repeat the same action twice. Indeed, the agent will not have the opportunity to repeat the right action twice because the episode would have ended after the first time it tried it. Furthermore, trying a wrong action twice is also not optimal as in incurs addition negative reward. Therefore, we can limit our search space to policies that try each action once. These policies differ by the ordering in which they try each one of these $d$ labels.

At the beginning of every episode, a image $x$ is sampled uniformly at random among all training images and its true label $y$ (during that episode) is sampled from $p(y|x, \gD)$.  Let $\pi$ be policy that tries each of the $d$ actions exactly once in its first $d$ trials. Let $T^{\pi}_{y}$ denotes the time when policy $\pi$ tries action $y$. Note that $(T^{\pi}_{y})_{y\in \gY}$ is a permutation. When the label chosen is $y$, the cumulative reward of $\pi$ for that episode is given by $r = \frac{\gamma^{ T_y^\pi}-1}{1-\gamma}$ and the expected cumulative reward (across episodes)  is given by:

\begin{equation}
    \begin{aligned}
    J(\pi) := \sum_{y \in \gY} p(y|x, \gD) \frac{\gamma^{ T_y^\pi}-1}{1-\gamma} = \frac{1}{1-\gamma}\left(\sum_{y \in \gY} p(y|x, \gD) \gamma^{ T_y^\pi} - 1\right)
    \end{aligned}
\end{equation}

From that expression, we see that in order to maximize its expected cumulative reward, a policy $\pi$ has to maximize $\sum_{y \in \gY} p(y|x, \gD) \gamma^{ T_y^\pi}$ which can be interpreted as the dot product of the vector $[p(y|x, \gD)]_{y \in \gY}$ and $[\gamma^{ T_y^\pi}]_{y \in \gY}$. By the rearrangement inequality, we know that this dot product is maximized when the components of the vectors are arranged in the same ordering. 

If we denote by $y_{(1)}, \dots y_{(d)}$ be the labels sorted in order of probability under the belief distribution: $p(y_{(1)} | x, \gD) \geq p(y_{(2)} | x, \gD) \geq \dots \geq p(y_{(d)} | x, \gD)$. Since $0<\gamma<1$ the rearrangement inequality implies that 
 the expected return is maximized when $T_{y_{(t)}}^\pi = t$. This corresponds to a policy that tries the labels sequentially from the most likely to the least likely.

\subsubsection{Memoryless policy}
\label{app:memorylessPolicy}
In this section, we will derive the optimal memoryless policy. Consider a memoryless policy that takes actions according to the distribution $\pi(\cdot|x)$ for the image $x$. When the true label is $y$ for the image $x$, the number of incorrect guesses is distributed as $\text{Geom}(p=\pi(y|x))$.

When the agent guesses correctly the label $y$ at the $t-$th guess then the cumulative reward is given by $r= - \frac{1-\gamma^t}{1-\gamma}$.  This happens with probability $(1-\pi(y\mid x))^t \times \pi(y \mid x)$. The expected return for policy $\pi$ evaluated on image $x$ is then given by: 

\begin{equation}
\begin{aligned}   
J(\pi|x) &= - \sum_{y\in\gY} \sum_{t=0}^\infty (1-\pi(y|x))^t \pi(y|x) \frac{1 - \gamma^t}{1-\gamma} p(y|x, \gD) \\
&=\sum_{y\in\gY} p(y|x, \gD) \frac{\pi(y \mid x)-1}{1 - \gamma(1-\pi(y|x))}
\end{aligned}
\end{equation}

When $\gamma=0$ (supervised learning problem), $ J(\pi) = \sum_{y\in\gY} p(y|x, \gD) \left(\pi(y \mid x)-1\right)$ is a linear function of $\pi$ and as expected, the optimal policy is to deterministically choose the label with the highest probability:$\pi^{*}(y \mid x) = 1\left[y = \argmax_{y\in\gY} p(y|x, \gD)\right]$.

When $\gamma>0$, the optimal policy is the solution to a constrained optimization problem that can be solved with Lagrange multipliers. When $\gamma=1$, the optimal policy can be written explicitly as: 
\begin{equation}
    \pi^{*} (y\mid x) = \frac{1}{\lambda} \sqrt{p(y|x, \gD)}
\end{equation}
where $\lambda$ is a normalization constant.

\section{Theoretical Results}
\epistemicpomdpdefn*
\begin{proof}
This proposition follows directly from the definition of the epistemic POMDP. If the MDP $\gM$ is sampled from $\gP(\gM)$ and $\gD$ is witnessed, then the posterior distribution over MDPs is given by $\gP(\gM | \gD)$, and the expected test-time return of $\pi$ given the evidence is 
\[\E_{\gM \sim \gP(\gM)}[J_\gM(\pi) | \gD] \coloneqq \E_{\gM \sim \gP(\gM | \gD)}[J_{\gM}(\pi)].\]

In the epistemic POMDP, where an episode corresponds to randomly sampling an MDP from $\gP(\gM | \gD)$, and a single episode being evaluated in this MDP, the expected return can be expressed identically:

\[J_{\gM^\po}(\pi) \coloneqq \E_{\gM \sim \gP(\gM | \gD)}[\E_{\pi, \gM}[\sum_{i=0}^\infty \gamma^t r(s_t, a_t)]] = \E_{\gM \sim \gP(\gM | \gD)}[J_{\gM}(\pi)].\]

\end{proof}
\label{appendix:Sec5}
\subsection{Optimal MDP Policies can be Arbitrarily Suboptimal}

\begin{restatable}{proposition}{dominatedRandom}
\label{prop:deterministic-dominated-random}
Let $\eps > 0$. There exists posterior distributions $\gP(\gM | \gD)$ where a deterministic Markov policy $\pi$ is optimal with probability at least $1-\epsilon$,
\vspace{-0.5em}
\begin{equation}
P_{\gM \sim \gP(\gM|\gD)}\left(\pi \in \argmax_{\pi'} J_\gM(\pi')\right) \geq 1-\eps,\end{equation}
\vspace{-0.5em}
but is outperformed by a uniformly random policy in the epistemic POMDP: $J_{\gM^\po}(\pi) < J_{\gM^\po}(\pi_{\text{unif}})$.
\end{restatable}
\begin{proof}
Consider two deterministic MDPs, $\gM_A$, and $\gM_B$ that both have two states and two actions: ``stay'' and ''switch''.
In both MDPs, the reward for the ``stay'' action is always zero.
In~$\gM_A$ the reward for ``switch'' is always 1, while in~$\gM_B$ the reward for ``switch'' is $-c$ for~$c>0$.
The probability of being in~$\gM_B$ is~$\epsilon$ while the probability of being in~$\gM_A$ is~$1-\epsilon$.
Clearly, the policy ``always switch'' is optimal in~$\gM_A$ and so is~$\epsilon$-optimal under the distribution on MDPs.
The expected discounted reward of the ``always switch'' policy is:
\begin{align}
    J(\pi_{\text{always switch}}) &= (1-\epsilon)\frac{1}{1-\gamma} - \epsilon\frac{c}{1-\gamma}%
    = \frac{1}{1-\gamma}(1-(c+1)\epsilon)\,.
\end{align}
On the other hand, we can consider a policy which selects actions uniformly at random.
In this case, the expected cumulative reward is %
\begin{align}
    J(\pi_{\text{random}}) &= (1-\epsilon)\frac{1}{2}\frac{1}{1-\gamma} - \epsilon \frac{c}{2}\frac{1}{1-\gamma}
    = \frac{1}{2}\frac{1}{1-\gamma}(1 - (c+1)\epsilon) = \frac{1}{2}J(\pi_{\text{always switch}})\,.
\end{align}
Thus for any~$\epsilon$ we can find a~${c > \frac{1}{\epsilon}-1}$ such that both policies have negative expected rewards and we prefer the random policy for being half as negative.
\end{proof}

\label{appendix:prop51}
\subsection{Bayes-optimal Policies May Take Suboptimal Actions Everywhere}

We formalize the remark that optimal policies for the MDPs in the posterior distribution may be poor guides for determining what the Bayes-optimal behavior is in the epistemic POMDP. The following proposition shows that there are epistemic POMDPs where the support of actions taken by the MDP-optimal policies is disjoint from the actions taken by the Bayes-optimal policy, so no method can ``combine'' the optimal policies from each MDP in the posterior to create Bayes-optimal behavior. 
\begin{proposition}
\label{prop:suboptimal-actions-everywhere}
There exist posterior distributions $\gP(\gM | \gD)$ where the support of the Bayes-optimal memoryless policy $\pi^{*\po}(a|s)$ is disjoint with that of the optimal policies in each MDP in the posterior. Formally, writing $\supp(\pi(a|s)) = \{a \in \gA: \pi(a|s) > 0\}$, then $\forall \gM$ with $\gP(\gM | \gD) > 0$ and $\forall s$:

\[\supp(\pi^{*\po}(a|s)) \cap \supp(\pi_\gM^*(a|s))  = \emptyset\]
\end{proposition}

\begin{proof}
The proof is a simple modification of the construction in Proposition 5.1. Consider two deterministic MDPs, $\gM_A$, and $\gM_B$ with equal support under the posterior, where both have two states and three actions: ``stay'', ''switch 1'', and ``switch 2''. In both MDPs, the reward for the ``stay'' action is always zero. In~$\gM_A$ the reward for ``switch'' is always 1, while in~$\gM_B$ the reward for ``switch'' is $-2$. The reward structure for ``switch 2'' is flipped: in $\gM_A$, the reward for ``switch 2'' is $-2$, and in $\gM_B$, the reward is $1$. Then, the policy ``always switch'' is optimal in $\gM_A$, and the policy ``always switch 2'' is optimal in $\gM_B$. However, any memoryless policy that takes either of these actions receives negative reward in the epistemic POMDP, and is dominated by the Bayes-optimal memoryless policy ``always stay'', which achieves 0 reward.
\end{proof}

\label{appendix:suboptimalActions}
\subsection{MaxEnt RL is Optimal for a Choice of Prior}

We describe a special case of the construction of Eysenbach and Levine \citep{eysenbachMaxEnt}, which shows that maximum-entropy RL in a bandit problem recovers the Bayes-optimal POMDP policy in an epistemic POMDP similar to that described in the RL image classification task.

Consider the family of MDPs $\{\gM_k\}_{k \in [n]}$ each with one state and $n$ actions, where taking action $k$ in MDP $\gM_k$ yields zero reward and the episode ends, and taking any other action yields reward $-1$ and the episode continues. Effectively, $\gM_k$ corresponds to a first-exit problem with ``goal action'' $k$. Note that this MDP structure is exactly what we have for the RL image classification task for a single image. Also consider the surrogate bandit MDP $\hat{\gM}$, also with one state and $n$ actions, but in which taking action $k$ yields reward $r_k$ with immediate episode termination. The following proposition shows that running max-ent RL in $\hat{\gM}$ recovers the optimal memoryless policy in a particular epistemic POMDP supported on $\{\gM_k\}_{k \in [n]}$.

\begin{proposition}
Let $\pi^* = \arg\max_{\pi \in \Pi} J_{\hat{\gM}}(\pi) + \gH(\pi)$ be the max-ent solution in the surrogate bandit MDP $\hat{\gM}$. Define the distribution $\gP(\gM|\gD)$ on $\{\gM_k\}_{k \in [n]}$ as $\gP(\gM_k|\gD) = \frac{\exp(2r_k)}{\sum_{j} \exp(2r_j)}$. Then, $\pi$ is the optimal memoryless policy in the epistemic POMDP $\gM^\po$ defined by $\gP(\gM | \gD)$. 
\end{proposition}
\begin{proof}
See Eysenbach and Levine \citep[Lemma 4.1]{eysenbachMaxEnt}. The optimal policy $\pi^*$ is given by $\pi^*(a=k) = \frac{\exp(r_k)}{\sum_{j} \exp(r_j)}$. We know from Appendix~\ref{app:memorylessPolicy} that this policy is optimal for epistemic POMDP $\gM^{\po}$ when $\gamma=1$.
\end{proof}
If allowing time-varying reward functions, this construction can be extended beyond ``goal-action taking'' epistemic POMDPs to the more general ``goal-state reaching'' setting in an MDP, where the agent seeks to reach a specific goal state, but the identity of the goal state hidden from the agent \citep[Lemma 4.2]{eysenbachMaxEnt}. 

\label{appendix:MaxEntRL}
\subsection{Failure of MaxEnt RL and Uncertainty-Agnostic Regularizations}
We formalize the remark made in the main text that while the Bayes-optimal memoryless policy is stochastic, methods that promote stochasticity in an uncertainty-agnostic manner can fail catastrophically. We begin by explaining the significance of this result: it is well-known that stochastic policies can be arbitrarily sub-optimal in a single MDP, and can be outperformed by deterministic policies. The result we describe is more subtle than this: there are epistemic POMDPs where any attempt at being stochastic in an uncertainty-agnostic manner is sub-optimal, and \textit{ also} any attempt at acting completely deterministically is also sub-optimal. Rather, the characteristic of Bayes-optimal behavior is to be stochastic in \textit{some} states (where it has high uncertainty), and not stochastic in others, and a useful stochastic regularization method must modulate the level of stochasticity to calibrate with regions where it has high epistemic uncertainty.
\begin{proposition}
\label{prop:stochastic-dominated-random}
Let $\alpha > 0, c > 0$. There exist posterior distributions $\gP(\gM | \gD)$, where the Bayes-optimal memoryless policy $\pi^{*\po}$ is stochastic. However, every memoryless policy $\pi_s$ that is ``everywhere-stochastic'', in that $\forall s \in \gS: \gH(\pi_s(a| s)) > \alpha$,
can have performance arbitrarily close to the uniformly random policy:
\[ \frac{J(\pi_s) - J(\pi_\text{unif})}{J(\pi^{*\po}) - J(\pi_\text{unif})} < c \]
\end{proposition}
\begin{figure}
    \centering
    \includegraphics[width=0.5\linewidth]{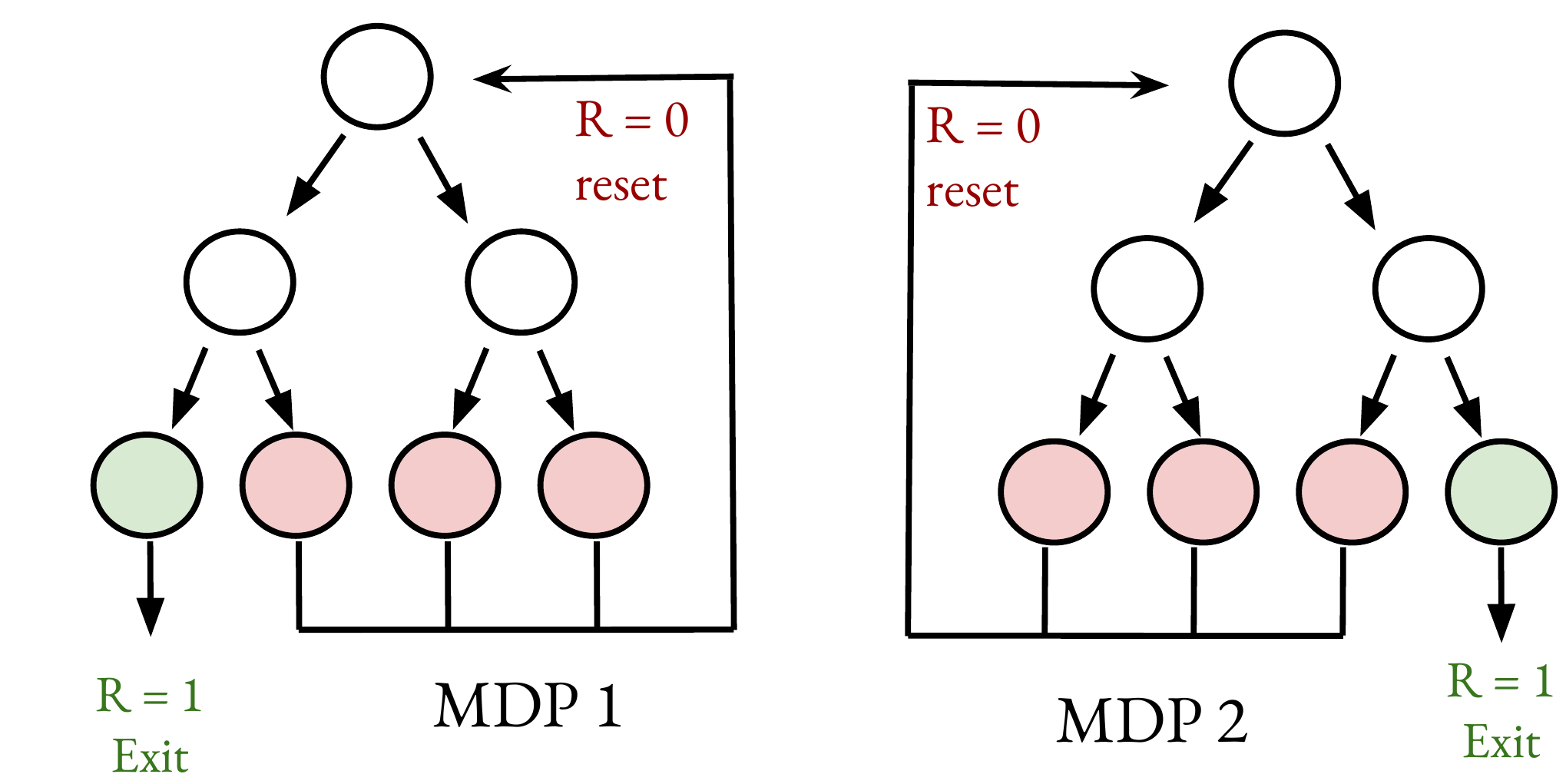}
    \caption{Visual description of Binary Tree MDPs described in proof of Proposition \ref{prop:stochastic-dominated-random} with depth $n=3$.}
    \label{fig:stochastic_illustration}
\end{figure}
\begin{proof}

Consider two binary tree MDP with $n$ levels, $\gM_1$ and $\gM_2$. A binary tree MDP, visualized in Figure \ref{fig:stochastic_illustration}, has $n$ levels, where level $k$ has $2^k$ states. On any level $k < n$, the agent can take a ``left'' action or a ``right'' action, which transitions to the corresponding state in the next level. On the final level, if the state corresponds to the terminal state (in green), then the agent receives a reward of $1$, and the episode exits, and otherwise a reward of $0$, and the agent returns to the top of the binary tree. The two binary tree MDPs $\gM_1$ and $\gM_2$ are identical except for the final terminal state: in $\gM_1$, the terminal state is the left-most state in the final level, and in $\gM_2$, the terminal state is the right-most state. Reaching the goal in $\gM_1$ corresponds to taking the ``left'' action repeatedly, and reaching the goal in $\gM_2$ corresponds to taking the ``right'' action repeatedly. We consider the posterior distribution that places equal mass on $\gM_1$ and $\gM_2$, $\gP(\gM_1|\gD)=\gP(\gM_2|\gD)=\frac{1}{2}$. A policy that reaches the correct terminal state with probability $p$ (otherwise reset) will visit the initial state a $\text{Geom}(p)$ number of times, and writing $\bar{\gamma} \coloneqq \gamma^n$, will achieve return $\frac{\bar{\gamma}p}{1 - \bar{\gamma} + p\bar{\gamma}} = \frac{1}{1 + \frac{1}{p} \frac{1-\bar{\gamma}}{\bar{\gamma}}}$.

\textit{Uniform policy:} A uniform policy randomly chooses between ``left'' and ``right'' at all states, and will reach all states in the final level equally often, so the probability it reaches the correct goal state is  $\frac{1}{2^n}$. Therefore, the expected return is $J(\pi_{\text{unif}}) = \frac{1}{1 + 2^n\frac{1-\bar{\gamma}}{\bar{\gamma}}}$.

\textit{Bayes-optimal memoryless policy:} The Bayes-optimal memoryless policy $\pi^{*\po}$ chooses randomly between ``left'' and ``right'' at the top level; on every subsequent level, if the agent is in the left half of the tree, the agent deterministically picks ``left'' and on the right half of the tree, the agent deterministically picks ``right''. Effectively, this policy either visits the left-most state or the right-most state in the final level. The Bayes-optimal memoryless policy returns to the top of the tree a $\text{Geom}(p=\frac{1}{2})$ number of times, and the expected return is given by $J(\pi^{*\po}) = \frac{1}{1 + 2\frac{1-\bar{\gamma}}{\bar{\gamma}}}$.

\textit{Everywhere-stochastic policy:} Unlike the Bayes-optimal policy, which is deterministic in all levels underneath the first, an everywhere-stochastic policy will sometimes take random actions at these lower levels, and therefore can reach states at the final level that are neither the left-most or right-most states (and therefore always bad). We note that if $\gH(\pi(a|s)) > \alpha$, then there is some $\beta > 0$ such that $\max_a \pi(a|s) < 1 - \beta$. For an $\alpha$-everywhere stochastic policy, the probability of taking at least one incorrect action increases as the depth of the binary tree grows, getting to the correct goal at most probability $\frac{1}{2}(1-\beta)^{n-1}$. The maximal expected return is therefore $J(\pi_s) \leq \frac{1}{1 + 2(\frac{1}{1-\beta})^{n-1} \frac{1-\bar{\gamma}}{\bar{\gamma}}}$

\[J(\pi^{*\po}) = \frac{1}{1 + 2\frac{1-\bar{\gamma}}{\bar{\gamma}}} ~~~~~~ J(\pi_s) = \frac{1}{1 + 2(\frac{1}{1-\beta})^{n-1} \frac{1-\bar{\gamma}}{\bar{\gamma}}} ~~~~~~J(\pi_{\text{unif}}) = \frac{1}{1 + 2^n\frac{1-\bar{\gamma}}{\bar{\gamma}}}\]

As $n\to \infty$, $J(\pi^{*\po}), J(\pi_s)$ and $J(\pi_{\text{unif}})$ will converge to zero. Using asymptotic analysis we can determine their speed of convergence and find that: 

\[J(\pi^{*\po}) \sim \frac{\bar{\gamma}}{2} ~~~~~~ J(\pi_s) \sim \frac{\bar{\gamma}}{ 2(\frac{1}{1-\beta})^{n-1}} ~~~~~~J(\pi_{\text{unif}}) \sim \frac{\bar{\gamma}}{ 2^n}\]

Using these asymptotics, we find that 

\[ \frac{J(\pi_s) - J(\pi_\text{unif})}{J(\pi^{*\po}) - J(\pi_\text{unif})} \sim \frac{1}{ (\frac{1}{1-\beta})^{n-1}} = (1-\beta)^{n-1}, \]

which shows that this ratio can be made arbitrarily small as we increase $n$. \qedsymbol

\textit{An aside: deterministic policies} While this proposition only discusses the failure mode of stochastic policies, \textit{all} deterministic memoryless policies in this environment also fail. A deterministic policy $\pi_d$ in this environment continually loops through one path in the binary tree repeatedly, and therefore will only ever reach one goal state, unlike the Bayes-optimal policy which visits both possible goal states. The best deterministic policy then either constantly takes the ``left'' action (which is optimal for $\gM_1$), or constantly takes the ``right'' action (which is optimal for $\gM_2$). Any other deterministic policy reaches a final state that is neither the left-most nor the right-most state, and will always get $0$ reward. The expected return of the optimal deterministic policy is $J(\pi_d) = \frac{\bar{\gamma}}{2}$, receiving $\bar{\gamma}$ reward in one of the MDPs, and $0$ reward in the other. When the discount factor $\gamma$ is close to $1$, the maximal expected return of a deterministic policy is approximately $\frac{1}{2}$, while the expected return of the Bayes-optimal policy is approximately $1$, indicating a sub-optimality gap.

\end{proof}
\label{appendix:CatastrophicFailure}
\subsection{Proof of Theorem 6.1}

\theoremLowerBound*
\begin{proof}
Before we begin, we recall some basic tools from analysis of MDPs. For a memoryless policy $\pi$, the state-action value function  $Q^\pi(s,a)$ is given by ${Q^\pi(s, a) = \E_\pi[\sum_{t \geq 0} \gamma^t r(s_t, a_t) | s_0 = s, a_0=a]}$. The advantage function $A^\pi(s, a)$ is defined as ${A^\pi(s, a) = Q^\pi(s, a) - \E_{a \sim \pi(\cdot|s)}[Q^\pi(s, a)]}$. The performance difference lemma \citep{kakadeLangford2002} relates the expected return of two policies $\pi$ and $\pi'$ in an MDP $\gM$ via their advantage functions as 

\begin{equation}
J_\gM(\pi') = J_{\gM}(\pi) + \frac{1}{1-\gamma}\E_{s \sim d_\gM^{\pi'}}[\E_{a \sim \pi'}[A_\gM^\pi(s ,a)]].
\end{equation}

We now begin the derivation of our lower bound:
\begin{equation}
\begin{aligned}
   J_{\hat{\gM}^{\po}}(\pi) &=\frac{1}{n} \sum_{i=1}^n J_{\gM_i}(\pi)\\
   &=\frac{1}{n} \sum_{i=1}^n  J_{\gM_i}(\pi_i) + \frac{1}{n} \sum_{i=1}^n  \left[J_{\gM_i}(\pi) -J_{\gM_i}(\pi_i) \right] \\
    &= \frac{1}{n} \sum_{i=1}^n  J_{\gM_i}(\pi_i) - \frac{1}{n(1-\gamma)} \sum_{i=1}^n  \mathbb{E}_{s\sim d_{\gM_i}^{\pi_i}}\left[ \mathbb{E}_{a\sim \pi_i}\left[A_{\gM_i}^{\pi}(s,a) \right]\right] \\
    &= \frac{1}{n} \sum_{i=1}^n  J_{\gM_i}(\pi_i) - \frac{1}{n(1-\gamma)} \sum_{i=1}^n  \mathbb{E}_{s\sim d_{\gM_i}^{\pi_i}}\left[ \mathbb{E}_{a\sim \pi_i}\left[A_{\gM_i}^{\pi}(s,a) \right]-\mathbb{E}_{a\sim \pi}\left[A_{\gM_i}^{\pi}(s,a) \right]\right] \\
\end{aligned}
\end{equation}
In the last equality we used the fact that $\mathbb{E}_{a\sim \pi}\left[A^{\pi}(s,a) \right] =0$. From there we proceed to derive a lower bound:
\begin{equation}
\begin{aligned}
    J_{\hat{\gM^\po}}(\pi) &= \frac{1}{n} \sum_{i=1}^n  J_{\gM_i}(\pi_i) - \frac{1}{n(1-\gamma)} \sum_{i=1}^n  \mathbb{E}_{s\sim d_{\gM_i}^{\pi_i}}\left[ \mathbb{E}_{a\sim \pi_i}\left[A_{\gM_i}^{\pi}(s,a) \right]-\mathbb{E}_{a\sim \pi}\left[A_{\gM_i}^{\pi}(s,a) \right]\right] \\
    &\geq \frac{1}{n} \sum_{i=1}^n  J_{\gM_i}(\pi_i) - \frac{2 r_{max}}{n(1-\gamma)^2} \sum_{i=1}^n \mathbb{E}_{s\sim d_{\gM_i}^{\pi_i}}\left[ D_{TV}\left(\pi_i(\cdot \mid s);\pi(\cdot \mid s)\right) \right] \\
    &\geq \frac{1}{n} \sum_{i=1}^n  J_{\gM_i}(\pi_i) - \frac{\sqrt{2} r_{max}}{(1-\gamma)^2n} \sum_{i=1}^n \mathbb{E}_{s\sim d_{\gM_i}^{\pi_i}}\left[ \sqrt{D_{KL}\left(\pi_i(\cdot \mid s)\,\, || \,\, \pi(\cdot \mid s)\right)} \right] \\
\end{aligned}
\end{equation}
where the first inequality is since $|A_{\gM_i}^\pi(s,a)| \leq \frac{r_{\max}}{1-\gamma}$ and the second from Pinsker's inequality. Our intention in this derivation is not to obtain the tighest lower bound possible, but rather to illustrate how bounding the advantage can lead to a simple lower bound on the expected return in the POMDP. The inequality can be made tighter using other bounds on $|A_{\gM_i}^\pi(s,a)|$, for example using $A_{\max} = \max_{i, s, a} |A_{\gM_i}^\pi(s,a)|$, or potentially a bound on the advantage that varies across state.
\end{proof}
\subsection{Proof of Proposition 6.1}
\propositionLinkFunction*

\begin{proof}
By Theorem~\ref{thm:lower_bound} we have that $\forall \alpha \geq \frac{\sqrt{2}r_{\text{max}}}{(1-\gamma)^2n}$:
\begin{equation}
J_{\hat{\gM^\po}}(f(\{\pi_i^*\})) \geq \frac{1}{n} \sum_{i=1}^n J_{\gM_i}(\pi_i^*) - \alpha \sum_{i=1}^n \E_{s \sim d_{\gM_i}^{\pi_i^*}}\left[\sqrt{ D_{KL}\left(\pi_i^*(\cdot | s)\,\, || \,\, f(\{\pi_i^*\})(\cdot | s)\right)}\right].
\end{equation}
Now, write $\pi'^* \in \argmax_\pi J_{\hat{\gM^\po}}(\pi)$ to be an optimal policy in the empirical epistemic POMDP, and consider the collection of policies $\{\pi'^*, \pi'^*, \dots, \pi'^*\}$.  Since $\{\pi_i^*\}$ is the optimal solution to Equation~\ref{eq:algo_objective}, we have
\begin{equation}
\begin{aligned}
J_{\hat{\gM^\po}}(f(\{\pi_i^*\})) &\geq \frac{1}{n} \sum_{i=1}^n J_{\gM_i}(\pi'^*) - \alpha \sum_{i=1}^n \E_{s \sim d_{\gM_i}^{\pi'^*}}\left[\sqrt{ D_{KL}\left(\pi'^*(\cdot | s)\,\, || \,\, f(\{\pi'^*\})(\cdot | s)\right)}\right]\\
&= \frac{1}{n} \sum_{i=1}^n J_{\gM_i}(\pi'^*)\\
&= J_{\hat{\gM^\po}}(\pi'^*),
\end{aligned}
\end{equation}
where the second line here uses the fact that $f(\pi'^*, \dots, \pi'^*) = \pi'^*$. Therefore $\pi^* \coloneqq f(\{\pi_i^*\})$ is optimal for the empirical epistemic POMDP.
\end{proof}

\section{Procgen Implementation and Experimental Setup}
\label{appendix:LEEP}

We follow the training and testing scheme defined by Cobbe et al. \citep{Cobbe2020LeveragingPG} for the Procgen benchmarks: the agent trains on a fixed set of levels, and is tested on the full distribution of levels. Due to our limited computational budget, we train on the so-called ``easy'' difficulty mode using the recommended $200$ training levels. Nonetheless, many prior work has found a significant generalization gap between test and train performance even in this easy setting, indicating it a useful benchmark for generalization \citep{Cobbe2020LeveragingPG, Raileanu2020AutomaticDA, Jiang2020PrioritizedLR}. We implemented LEEP on top of an existing open-source codebase released by Jiang et al. \citep{Jiang2020PrioritizedLR}. Full code is provided in the supplementary for reference.

LEEP maintains $n=4$ policies $\{\pi_i\}_{i \in [n]}$, each parameterized by the ResNet architecture prescribed by Cobbe et al. \citep{Cobbe2020LeveragingPG}. In LEEP, each policy is optimized to maximize the entropy-regularized PPO surrogate objective alongside a one-step KL divergence penalty between itself and the linked policy $\max_i \pi_i$; gradients are not taken through the linked policy. 

\[\E_{\pi_i}[\min(r_t(\pi)A^\pi(s,a), \text{clip}(r_t(\pi), 1-\eps, 1+\eps)A^\pi(s,a) + \beta \gH(\pi_i(a|s)) - \textcolor{purple}{\alpha D_{KL}(\pi_i(a|s) \| \max_j \pi_j(a|s))}] \]

The penalty hyperparameter $\alpha$ was obtained by performing a hyperparameter search on the Maze task for all the comparison methods (including LEEP) amongst $\alpha \in [0.01, 0.1, 1.0, 10.0]$. Since LEEP trains $4$ policies using the same environment budget as a single PPO policy, we change the number of environment steps per PPO iteration from $16384$ to $4096$, so that the PPO baseline and each policy in our method takes the same number of PPO updates. All other PPO hyperparameters are taken directly from \citep{Jiang2020PrioritizedLR}. 

In our implementation, we parallelize training of the policies across GPUs, using one GPU for each policy. We found it infeasible to run more ensemble members due to GPU memory constraints without significant slowdown in wall-clock time. Running LEEP on one Procgen environment for 50 million steps requires approximately 5 hrs in our setup on a machine with four Tesla T4 GPUs. 
\clearpage
\section{Procgen Results}
\label{appendix:morePlots}
\begin{figure}[H]
    \centering
    \includegraphics[width=\linewidth]{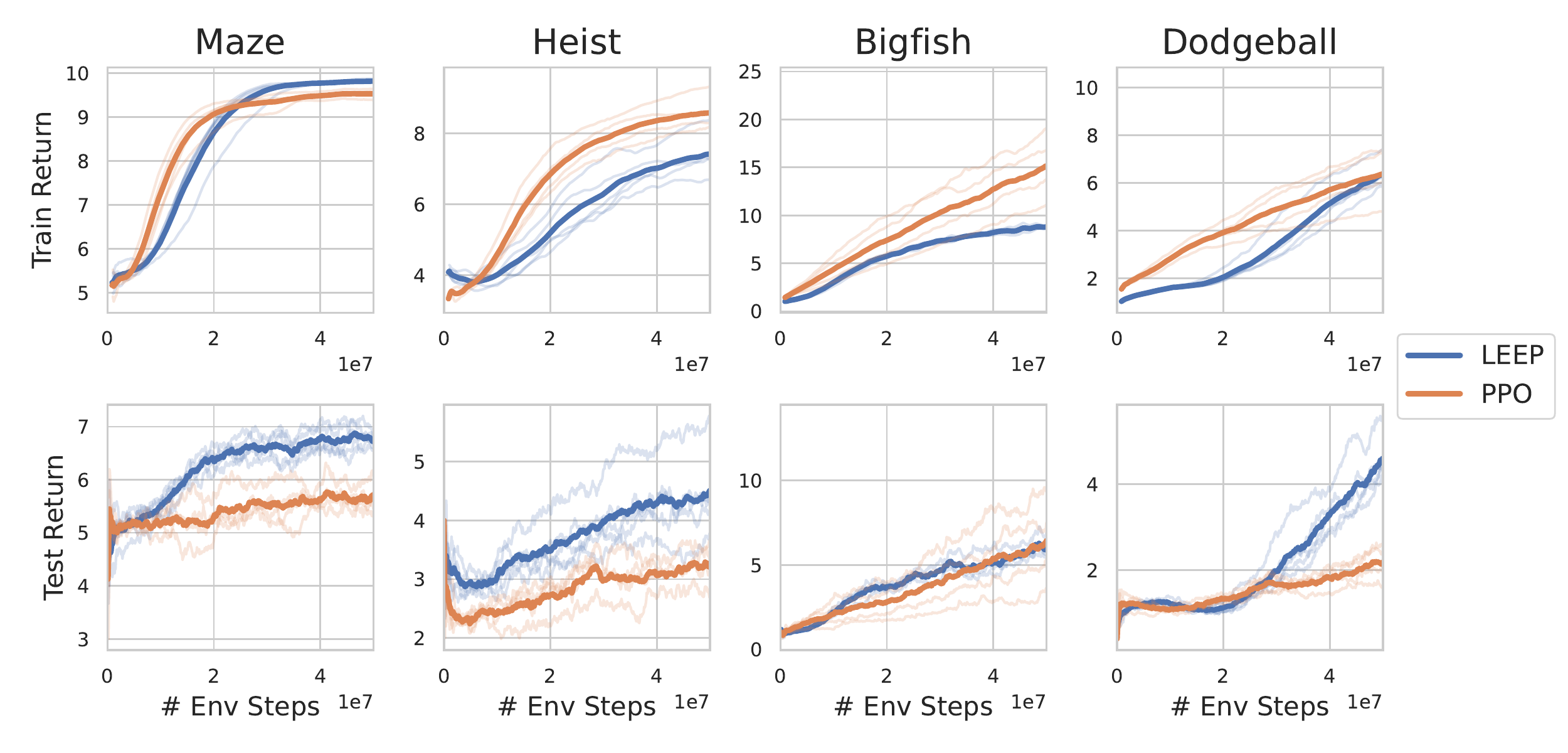}
    \caption{Training (top) and test (bottom) returns for LEEP and PPO on four Procgen environments. Results averaged across 5 random seeds. LEEP achieves equal or higher training return compared to PPO, while having a lower generalization gap between test and training returns.}
    \label{fig:appendix_all_procgen}
\end{figure}

\begin{figure}[H]
    \centering
    \includegraphics[width=0.8\linewidth]{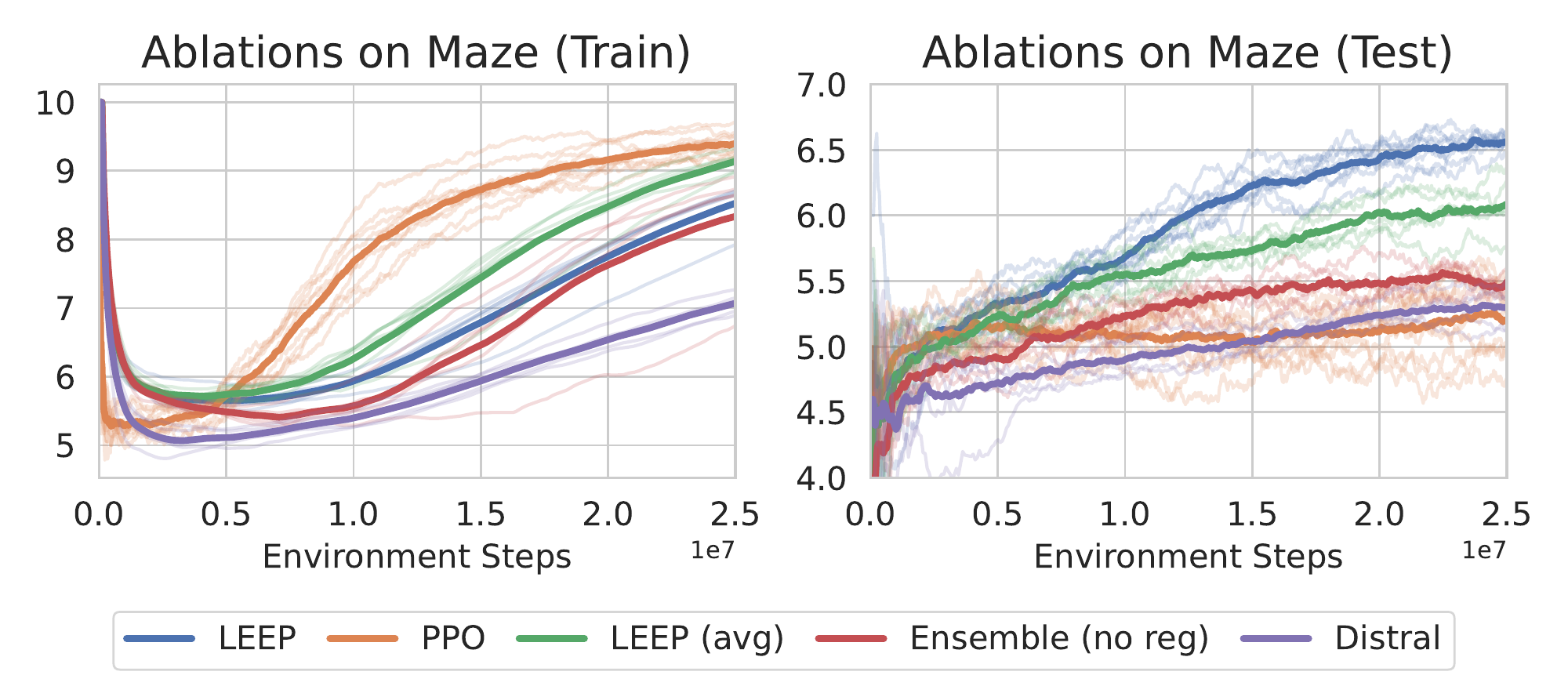}
\includegraphics[width=0.8\linewidth]{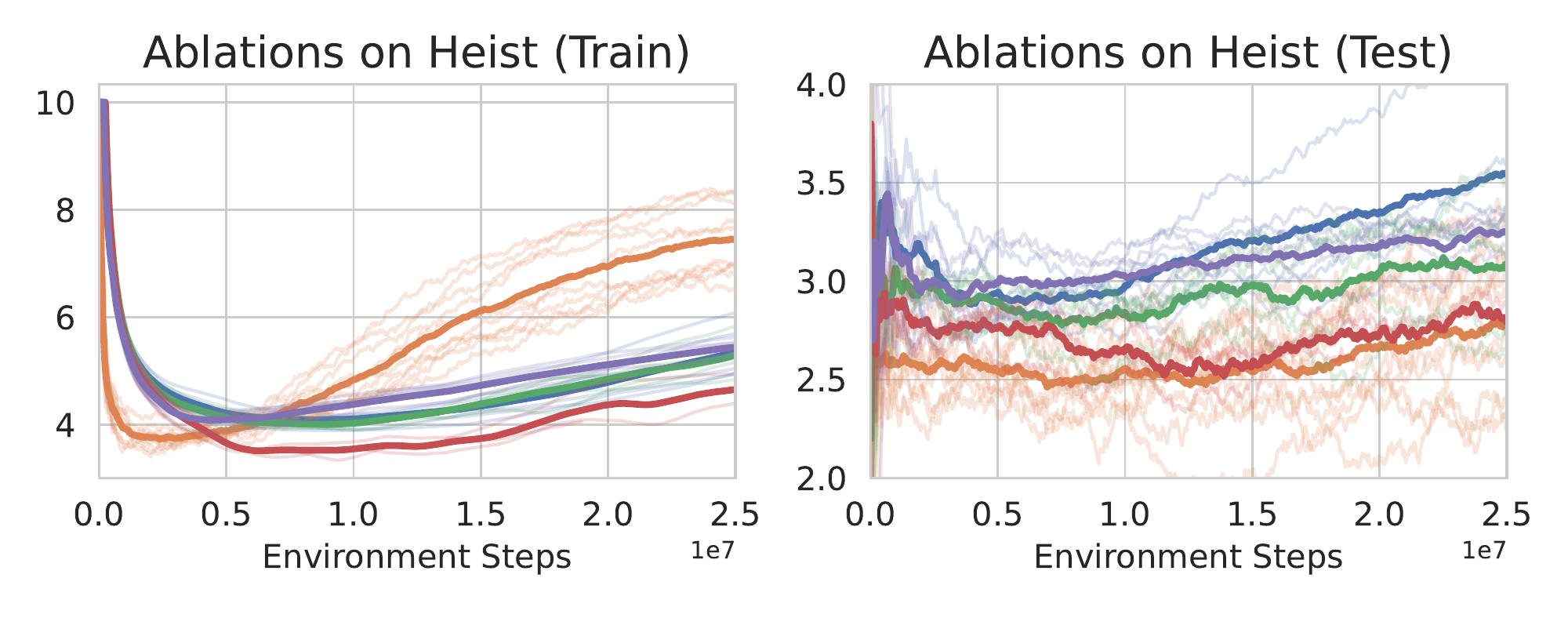}
    \caption{Training and test returns for various ablations and comparisons of LEEP.}
    \label{fig:appendix_procgen_ablations}
\end{figure}

\begin{figure}[H]
    \centering
    \includegraphics[width=\linewidth]{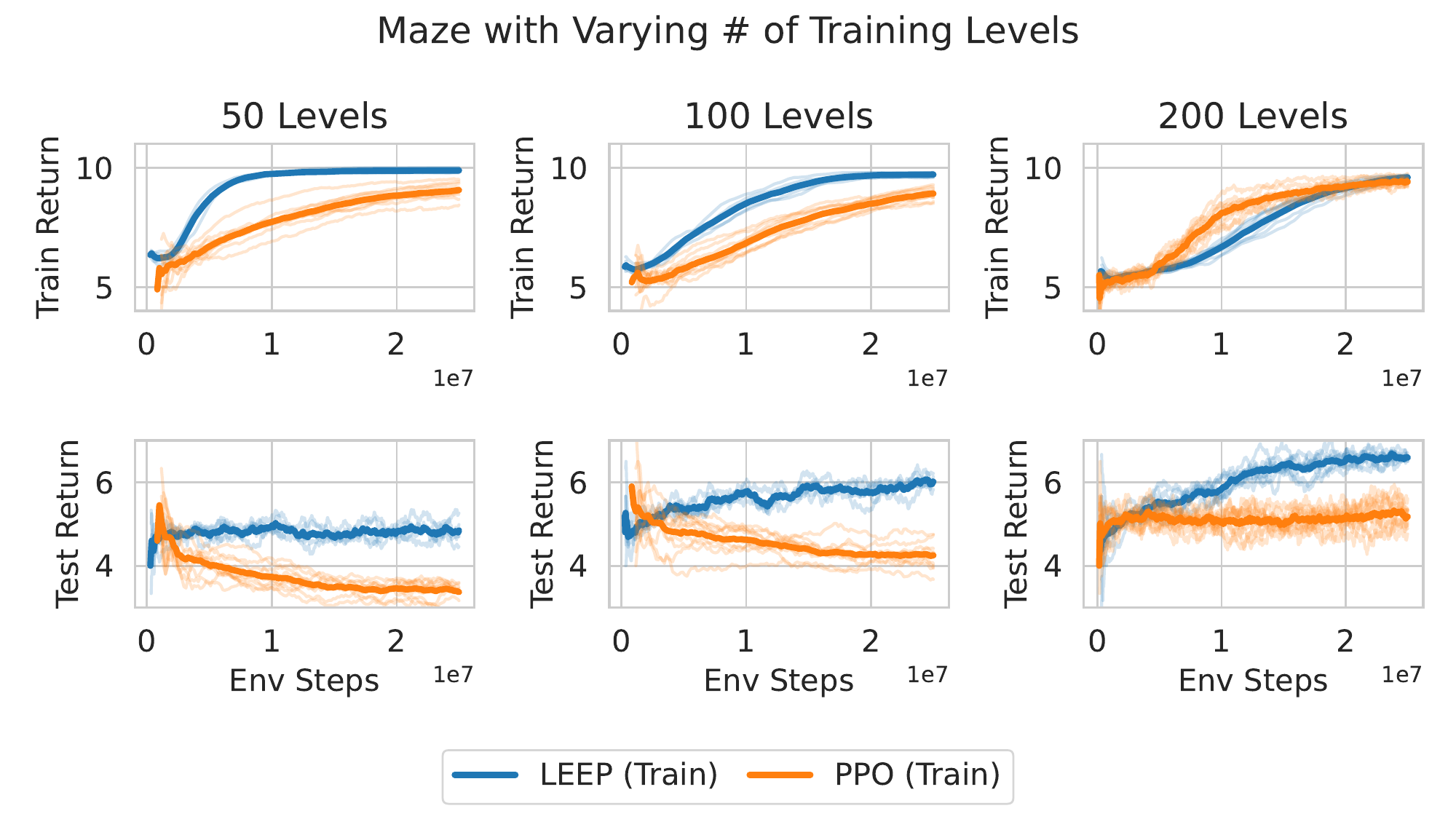}
    \caption{Performance of LEEP and PPO as the number of training levels provided varies. While the learned performance of the PPO policy is worse than a \textit{random policy} with less training levels, LEEP avoids this overfitting and in general, demonstrates a smaller train-test performance gap than PPO. }
    \label{fig:appendix_maze_varying_levels}
\end{figure}
\clearpage

\end{document}